\documentclass[10pt,journal]{IEEEtran}
\usepackage[ruled,linesnumbered]{algorithm2e}
\usepackage{amsthm}
\usepackage{amsmath}
\usepackage{slashbox}
\usepackage{graphicx}
\usepackage{epsfig,graphics,subfigure,psfrag,amsmath,amssymb}
\usepackage{cite}
\usepackage{float}
\usepackage{tocvsec2}
\usepackage{color}
\usepackage{verbatim}
\usepackage{mathtools}
\usepackage{hhline}
\usepackage{enumerate}
\usepackage{arydshln}
\usepackage{bbm}
\usepackage{stackengine}
\usepackage{mathrsfs}
\usepackage{enumitem}
\usepackage{caption}
\usepackage{stfloats}
\usepackage{amssymb}
\usepackage{booktabs}
\usepackage{makecell}
\usepackage[export]{adjustbox} 
\usepackage{hyperref}       
\usepackage{setspace}

\usepackage[symbol]{footmisc}

\newtheorem{thm}{Theorem}
\newtheorem{lem}{Lemma}

\newtheorem{rem}{Remark}

\newtheorem{assump}{Assumption}

\newtheorem{problem}{Problem}

\makeatletter

\newcommand{\Rmnum}[1]{\expandafter\@slowromancap\romannumeral #1@}
\makeatother

\makeatletter
\newcommand*\bigcdot{\mathpalette\bigcdot@{.5}}
\newcommand*\bigcdot@[2]{\mathbin{\vcenter{\hbox{\scalebox{#2}{$\m@th#1\bullet$}}}}}
\makeatother

\DeclarePairedDelimiter\floor{\lfloor}{\rfloor}

\def\@#1{\pmb{#1}}
\def\b#1{\mathbb{#1}}

\def\s#1{\mathsf{#1}}

\def\ca#1{\mathcal{#1}}

\def\v#1{\boldsymbol{#1}}

\newcommand\aleq{\mathrel{\stackrel{\makebox[0pt]{\mbox{\normalfont\tiny (a)}}}{\leq}}}
\newcommand\bleq{\mathrel{\stackrel{\makebox[0pt]{\mbox{\normalfont\tiny (b)}}}{\leq}}}

\newcommand\aeq{\mathrel{\stackrel{\makebox[0pt]{\mbox{\normalfont\tiny (a)}}}{=}}}

\DeclareMathOperator*{\argmin}{arg\,min}

\allowdisplaybreaks

\begin{document}
\title{{Federated Learning With Energy Harvesting Devices: An MDP Framework}}
\author{Kai Zhang$^{\dagger}$, Xuanyu Cao$^{\ddagger}$, \emph{Senior Member}, \emph{IEEE}, and Khaled B. Letaief$^{\dagger}$, \emph{Fellow}, \emph{IEEE}\\
	$^{\dagger}$Department of Electronic and Computer Engineering, The Hong Kong University of Science and Technology
	\\$^{\ddagger}$School of Electrical Engineering and Computer Science, Washington State University
	\\Email: kzhangbn@connect.ust.hk, xuanyu.cao@wsu.edu, eekhaled@ust.hk
}

\maketitle

\begin{abstract}

Federated learning (FL) necessitates that edge devices conduct local training and communicate with a parameter server, resulting in significant energy consumption. A key challenge in practical FL systems is the rapid depletion of battery-limited edge devices, which limits their operational lifespan and impacts learning performance. To tackle this issue, we implement energy harvesting techniques in FL systems to capture ambient energy, thereby providing continuous power to edge devices.
We first establish the convergence bound for the wireless FL system with energy harvesting devices, illustrating that the convergence is affected by partial device participation and packet drops, both of which depend on the energy supply. To accelerate the convergence, we formulate a joint device scheduling and power control problem and model it as a Markov decision process (MDP). By solving this MDP, we derive the optimal transmission policy and demonstrate that it possesses a monotone structure with respect to the battery and channel states. To overcome the curse of dimensionality caused by the exponential complexity of computing the optimal policy, we propose a low-complexity algorithm, which is asymptotically optimal as the number of devices increases. Furthermore, for unknown channels and harvested energy statistics, we develop a structure-enhanced deep reinforcement learning algorithm that leverages the monotone structure of the optimal policy to improve the training performance. Finally, extensive numerical experiments on real-world datasets are presented to validate the theoretical results and corroborate the effectiveness of the proposed algorithms.

\end{abstract}


	\section{Introduction}

Machine learning is a cornerstone technology driving artificial intelligence (AI) in diverse fields, such as autonomous driving \cite{kiran2021deep} and the Internet of Things \cite{khan2021federated }.
Traditional centralized machine learning requires the aggregation and processing of distributed data on a central server, which leads to increased data traffic and raises concerns about data privacy and security.
To address these issues, federated learning (FL) has emerged as a novel distributed machine learning paradigm that enables collaborative model training without uploading raw data to a central server\cite{letaief2019roadmap,letaief2021edge}.
In FL, edge devices (e.g., mobile phones and sensors) conduct local training on their respective datasets and subsequently exchange information (e.g., model parameters or gradients) with a parameter server.
Due to its distinctive architecture, FL can alleviate the computation and communication burdens and preserve the data privacy of involved edge devices \cite{yang2020federated,pillutla2022robust}.

In practical scenarios, the performance of FL is severely affected by radio resources, such as transmission power and bandwidth.
The limited radio resources among participating edge devices hinder data transmission and increase communication latency, thereby degrading the learning performance.
To tackle these issues, numerous papers have investigated communication resource management for wireless FL systems.
In \cite{9210812}, the authors investigated a joint radio resource allocation and user selection problem for wireless FL, aiming to minimize the training loss under delay and energy consumption constraints. A similar resource allocation problem of achieving the optimal tradeoff between convergence speed and power consumption for wireless FL was studied in \cite{9261995}. In \cite{sery2021over}, the authors developed a joint computation and transmission scheme for wireless FL under a more challenging setting with heterogeneous data. 
While the aforementioned studies explored static resource allocation according to the statistical properties of wireless channels, other research focused on the dynamic allocation of communication resources.
The authors of \cite{cao2021optimized} proposed a dynamic power control policy based on instant channel state information (CSI) to speed up the convergence of over-the-air FL. Similarly, a dynamic device scheduling and power control scheme for over-the-air FL was developed in \cite{sun2021dynamic}.
Bandwidth allocation for wireless FL has also been studied extensively.
In \cite{xu2020client}, the authors investigated a joint device selection and bandwidth allocation problem for wireless FL, devising an online algorithm to maximize the number of participating devices in each round. A similar problem incorporating latency constraints was studied in \cite{shi2020joint}. Additionally, the authors of \cite{wan2021convergence} studied a joint power and bandwidth allocation problem to simultaneously minimize the power consumption, computation cost, and time cost of wireless FL systems.

Despite the extensive research, all the aforementioned studies have either considered finite energy supply for edge devices or completely neglected the energy limitations.
However, merely utilizing the finite energy stored in the batteries of edge devices will lead to a reduced operational lifespan. Once the battery is depleted, edge devices cannot contribute to the training process of FL, which degrades the learning performance.
Additionally, frequent battery replacements and recharges are impractical, especially for devices deployed in remote or inaccessible environments.
To address this issue, energy harvesting (EH) FL has emerged as a promising solution. This technique collects ambient energy (e.g., solar, wind or thermoelectric energy) to continuously power edge devices, thereby extending their operational lifetime and improving the learning performance.
Furthermore, EH FL uses clean and renewable energy sources, making it more environmentally friendly compared to conventional FL with battery-powered edge devices.
Several recent studies have investigated EH FL.
In \cite{guler2021energy}, a distributed stochastic gradient descent (SGD) algorithm was developed under intermittent and heterogeneous energy arrivals, with provable convergence guarantees. The subsequent work \cite{guler2021framework} extended this algorithm by proposing a scalable client scheduling and model update scheme based on each client’s energy renewal cycle, achieving notable improvements over energy-agnostic baselines.
In \cite{shen_federated_2022}, a client scheduling algorithm was introduced to accelerate the convergence of FL over EH-enabled wireless networks.
A joint device selection and resource allocation problem was investigated in \cite{zeng_federated_2023} to reduce training latency and improve the data utility of EH FL. The authors of \cite{hamdi_federated_2022} extended this problem to a more sophisticated system, where the server employs massive multiple-input multiple-output (MIMO) to serve clients.
Exploiting the over-the-air computation technique, bandwidth-efficient transmission schemes tailored for EH FL were proposed in \cite{chen_energy_2022,aygun2022over}.
In \cite{liu2021age}, a joint client selection and bandwidth allocation scheme was designed for FL with EH devices, aiming to minimize the ages of local models and the latency of global iterations.
To accelerate the convergence of EH FL, the authors of \cite{an_online_2023} formulated a joint transceiver design and device scheduling problem and developed a Lyapunov-based online optimization algorithm to solve it.

The prior works \cite{guler2021energy, guler2021framework, shen_federated_2022, zeng_federated_2023, hamdi_federated_2022, chen_energy_2022, aygun2022over,liu2021age} on EH FL concentrate on designing transmission strategies for one typical time slot.
However, the overall performance of FL is affected by the cumulative model aggregation errors over the entire training period, necessitating careful utilization of harvested energy from a long-term perspective.
We note that \cite{an_online_2023} is an early attempt to optimize the long-term efficiency of EH FL.
However, its online optimization algorithm fails to exploit the statistical knowledge of system dynamics, e.g., the distributions of the channel gains and the amount of harvested energy, resulting in suboptimal learning performance.
Thus, in this paper, we are motivated to develop a Markov decision process (MDP) framework for designing the optimal transmission policy of EH FL.
The proposed MDP framework utilizes the statistical knowledge of channels and harvested energy to address their inherent stochastic nature, thereby enabling efficient utilization of harvested energy and enhanced learning performance.
The main contributions are summarized as follows.

1) We characterize the convergence bound of the wireless FL system with EH devices, demonstrating the impacts of partial device participation and packet drops, both of which further depend on the energy supply.
Based on the convergence bound, we formulate a joint device scheduling and power control problem to minimize the optimality gap between the expected and optimal global loss.

2) We propose an MDP framework for the formulated optimality gap minimization problem and compute the optimal transmission policy.
We prove that the optimal policy exhibits a monotone structure with respect to the battery and channel states.
Moreover, to overcome the curse of dimensionality caused by the exponential complexity of computing the optimal policy, we propose a low-complexity algorithm that is asymptotically optimal as the number of devices increases.

3) For a more general scenario with unknown channels and harvested energy statistics, we develop a structure-enhanced deep reinforcement learning (SE-DRL) algorithm that leverages the monotone structure of the optimal policy to improve the training performance.

4) We validate the theoretical results and evaluate the performance of the proposed algorithms on real-world datasets. Numerical experiments demonstrate that the proposed algorithms outperform benchmark schemes in terms of test accuracy. The impact of battery capacity is also investigated.

The rest of this paper is organized as follows. In Section II, we elaborate on the system model. In Section III, we conduct the convergence analysis and formulate the design problem accordingly. In Section IV, we develop an MDP framework and compute its optimal policy. In Section V, we propose an asymptotically optimal policy with low complexity. In Section VI, we develop an SE-DRL algorithm. Simulation results are presented in Section VII, and we conclude the paper in Section VIII.

\emph{Notations:} Bold lowercase letters stand for column vectors. The symbol $\b{R}$ denotes the set of real numbers.
$[n]$ denotes the set $\left\lbrace 1,2,\ldots,n\right\rbrace $.
$\b{E}[\cdot]$ denotes the expectation operation. $\nabla$ represents the gradient operator. $|\cdot|$ and $\|\cdot\|$ stand for the $\ell_1$ and $\ell_2$ norms of vectors. $\mathbbm{1}_{\left\lbrace \cdot \right\rbrace }$ denotes the indicator function. $\b{P}(\cdot|\cdot)$ denotes the condition probability. For a set $\ca{X}$, $\ca{X}^n$ stands for the $n$-fold Cartesian product. The floor function, denoted by $\floor{x}$, maps a real number $x$ to the largest integer less than or equal to $x$.

\section{System Model}

\subsection{Federated Learning Model}

\begin{figure}[tbp]
	\renewcommand\figurename{\small Fig.}
	\centering \vspace*{4pt} \setlength{\baselineskip}{10pt}
	\includegraphics[width = 0.49\textwidth,center,trim=16 40 16 70,clip]{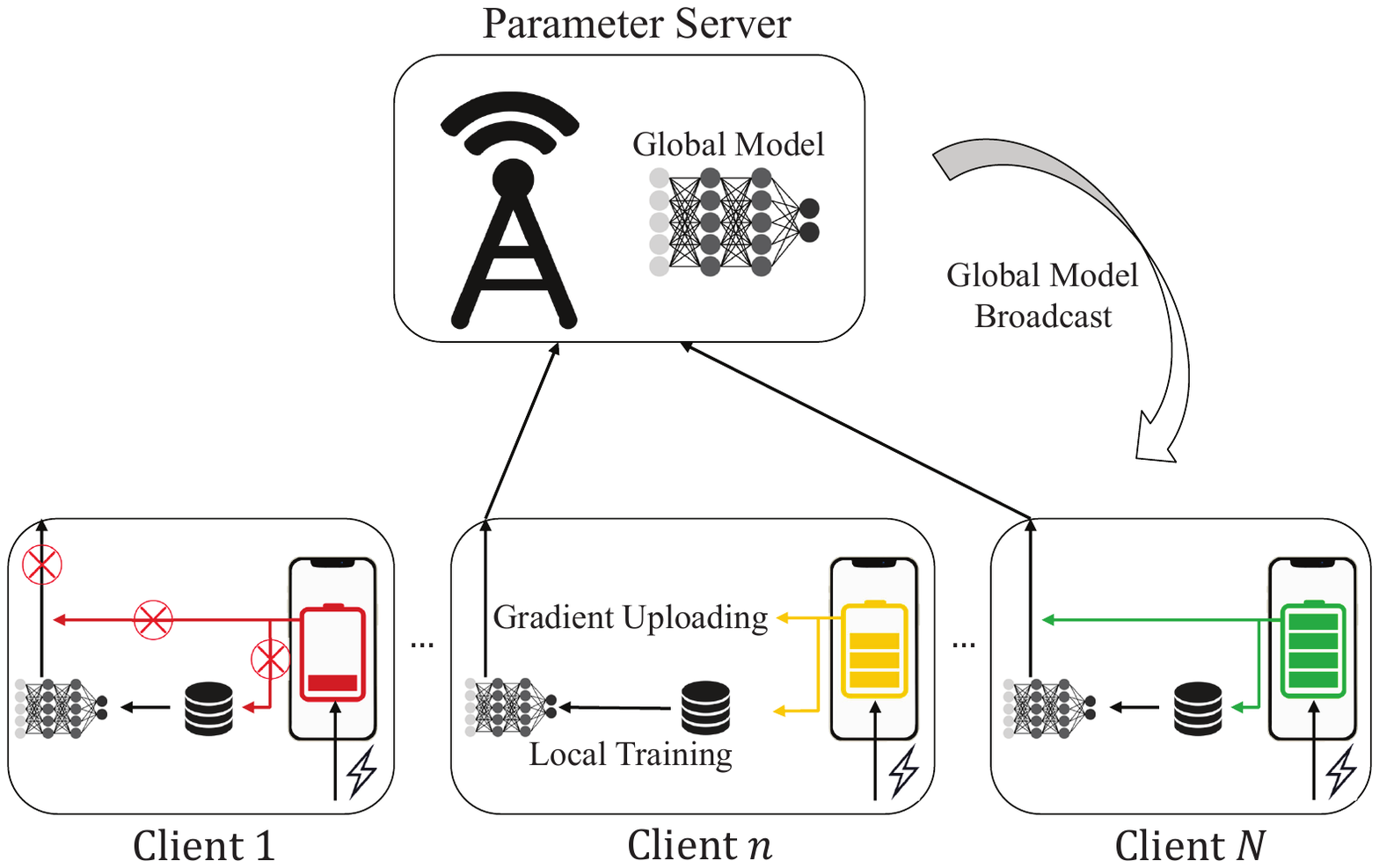}
	\vspace{-15pt}
	\caption{An illustration of the wireless FL system with EH devices.}\label{system_model_fig}
\end{figure}

We consider a wireless FL system, where $N$ edge devices collaboratively train a common learning model $\v{w}=[w_1,\ldots,w_d]\in \b{R}^{d}$ under the coordination of a parameter server, as shown in Fig. 1.
Each device $n$ is associated with a local loss function $F_n(\v{w})=  \b{E}_{\xi_n\sim \ca{D}_{n}} \left[  f_n(\v{w}; \xi_n )\right] $, where $f_n(\v{w}; \xi_n)$ denotes the sample-wise loss function that quantifies the loss of learning model $\v{w}$ on the data sample $\xi_n$, and $\ca{D}_n$ is the local dataset of device $n$.
%
%
The goal of the FL system is to train an optimal global model $\v{w}^{*}$ that minimizes the global loss function $F(\v{w})$, i.e.,
\begin{equation}\label{global_fl_problem}
	\begin{aligned}
		\min_{\v{w}} \; F(\v{w})=\frac{1}{N}\sum_{n=1}^{N} F_n(\v{w}).
	\end{aligned}    
\end{equation}
Problem \eqref{global_fl_problem} is known as empirical risk minimization, which arises in many ML problems. To solve the learning problem \eqref{global_fl_problem} in a distributed manner, the FL system conducts the following three steps iteratively under the coordination of the server:
\begin{itemize}
	\item  The server schedules a subset of edge devices and broadcasts the current global model to them.
	\item  Each scheduled device performs multiple steps of local SGD on its private dataset, and then uploads its local model update to the server.
	\item  The server aggregates the local updates received from the participating devices and then refines the global model.
\end{itemize}
Each iteration consisting of these three steps is called a time slot, indexed by $t$. The whole training process continues until reaching a preset maximum number of time slots $T$.

\subsection{Local Computation Model}
At time slot $t$, the server first schedules a subset of edge devices $\Omega_t \subseteq [N]$ and then broadcasts the current global model $\v{w}_{t}$ to these selected devices.
We assume that the server can broadcast the global model to edge devices without any error, due to its strong communication capability (e.g., large transmission power). 
Let $\beta_{n,t}\in\left\lbrace 0,1\right\rbrace $ denote the scheduling indicator, with $\beta_{n,t}=1$ if device $n$ is scheduled at time slot $t$, and $\beta_{n,t}=0$ otherwise.
Then, we have $\Omega_t=\left\lbrace n \,\vert \, \beta_{n,t}=1, n \in [N] \right\rbrace $.
After receiving the global model $\v{w}_{t}$, each scheduled device $n \in \Omega_t$ initialize the model for local SGD by setting \(\v{w}_{n,t}^{(1)} = \v{w}_{t}\). Then, for each local iteration \( k \in \{1,2, \ldots, K\} \), device $n$ samples a mini-batch $\vspace{0pt}$\( \hspace{2pt} \v{\xi}_{n,t}^{(k )} \hspace{0.5pt} \subseteq \hspace{0.5pt} \mathcal{D}_i \) from its local dataset and computes the stochastic gradient as follows:
\begin{equation}\label{eq:local_gradient}
	\v{g}_{n,t}^{(k)} = \nabla f_n(\v{w}_{n,t}^{(k)}; \v{\xi}_{n,t}^{(k)}).
\end{equation}
The local model is then updated as follows:
\begin{equation}\label{eq:local_update}
	\v{w}_{n,t}^{(k+1)} = \v{w}_{n,t}^{(k)} - \eta \, \v{g}_{n,t}^{(k)},
\end{equation}
where \( \eta > 0 \) denotes the learning rate.
The local computation energy consumption for device $n$ at time slot $t$ can be expressed as \cite{yang2020energy}
\begin{equation}\label{eq:energy_consumption}
	e_{n,t}^{\text{cmp}} = \beta_{n,t} K  I_n^2  c_n  | \v{\xi}_{n} |,
\end{equation}
where \( I_n \) is the CPU cycle frequency of device \( n \), \( c_n \) denotes the number of CPU cycles required to process one data sample, and \(  | \v{\xi}_{n} | \) is the (constant) mini-batch size used in each local iteration.


\subsection{Communication Model}
After local computation, each scheduled device $n$ send its local model update $\Delta\v{w}_{n,t}=\v{w}^{(K+1)}_{n,t} - \v{w}_{t}$ to the server for global aggregation. For the uplink transmission, we adopt an orthogonal frequency division multiple access (OFDMA) technique, where each device occupies one resource block to upload its local update. We consider a non-trivial scenario where no more than $R<N$ devices can upload their local updates at one time slot due to the limited bandwidth constraint. Thus, we have
\begin{equation}\label{bandwidth_constraint}
	\begin{aligned}
		\sum_{n=1}^{N}\beta_{n,t}\leq R, \forall t.
	\end{aligned}
\end{equation}
We assume that the uplink channel is block fading, i.e., the channel fading keeps constant during one time slot.
Denote $\tilde{h}_{n,t}$ as the uplink channel from device $n$ to the server at time slot $t$ with its channel gain $h_{n,t}=|\tilde{h}_{n,t}|$.
The uplink rate of device $n$ at time slot $t$ is given by
\begin{equation}\label{up_rate}
	\begin{aligned}
		r_{n,t}=W\log_2 \left(1+\frac{p_{n,t} h_{n,t}}{N_0 W}\right),
	\end{aligned}    
\end{equation}
where $p_{n,t}$ is the transmission power of device $n$ at time slot $t$, $N_0$ is the noise power spectral density, and $W$ is the bandwidth for each resource block.
Then, the communication energy consumption of device $n$ at time slot $t$ can be derived as
\begin{align}
	\label{com_energy}
	&e_{n,t}^{\text{com}}=\frac{p_{n,t} s_n}{r_{n,t}},
\end{align}
where $s_n$ denotes the size of the $n$-th device's local model update. The total energy consumption of device $n$ at time slot $t$ is given by
\begin{equation}\label{total_energy}
	\begin{aligned}
		e_{n,t}=e_{n,t}^{\text{cmp}}+e_{n,t}^{\text{com}}.
	\end{aligned}    
\end{equation}

In this paper, we assume that the local model update of each device is transmitted as a single packet in the uplink. Due to the limited bandwidth and transmission power of edge devices, the packet error rate of uplink transmission cannot be ignored.
Specifically, the packet error rate of the uplink transmission from device $n$ to the server at time slot $t$ is given by \cite{xi2011general}
\begin{equation}\label{per}
	\begin{aligned}
		q_{n,t}(h_{n,t},p_{n,t})=1-\exp\left(-\frac{\sigma N_0 W}{p_{n,t} h_{n,t}}\right),
	\end{aligned}    
\end{equation}
where $\sigma$ is a constant parameter, namely, the waterfall threshold \cite{xi2011general}.
Let $\zeta_{n,t}\in\left\lbrace 0,1\right\rbrace$ denote the indicator of successful transmission, where $\zeta_{n,t}=0$ indicates that the received local update of device $n$ at time slot $t$ is dropped due to the occurrence of packet error, and $\zeta_{n,t}=1$ otherwise. Then, we have
\begin{equation}\label{succ_index}
	\begin{aligned}
		\zeta_{n,t}= \begin{cases}1, & \text { with probability } 1-q_{n,t}(h_{n,t},p_{n,t}); \\ 0, & \text { with probability } q_{n,t}(h_{n,t},p_{n,t}).\end{cases}
	\end{aligned}    
\end{equation}
In the considered wireless FL system, the server employs successfully received local updates from the scheduled devices to refine the global model as follows:
\begin{equation}\label{update_global}
	\begin{aligned}
		\v{w}_{t+1}=\v{w}_{t}+\frac{ \sum_{n=1}^{N}\beta_{n,t}\zeta_{n,t}\Delta\v{w}_{n,t}}{\sum_{n=1}^{N}\beta_{n,t}\zeta_{n,t}}
	\end{aligned}    
\end{equation}
If all devices lack sufficient energy to upload local gradients to the server, or if all transmitted local gradients are lost, the server can utilize the previous global model, i.e., $\v{w}_{t+1}=\v{w}_{t}$.

\subsection{Energy Harvesting and Update Model}

The limited battery capacity of edge devices poses significant challenges in sustaining the long-term operation of the wireless FL system without battery recharging.
To address this issue, we consider that each edge device is equipped with an EH device to extract energy from the ambient environment.
Furthermore, we assume that each device $n$ is equipped with a rechargeable battery with limited capacity $b^{\max}$. The energy harvested at time slot $t$ is first stored in the battery and then available at the next time slot $t+1$. Let $u_{n,t}$ and $b_{n,t}$ denote the harvested energy and the energy level of device $n$ at time slot $t$, respectively.
The energy level dynamics of device $n$ over two adjacent time slots is given by
\begin{equation}\label{evo_energy}
	\begin{aligned}
		b_{n,t+1}=\min \{b_{n,t}+u_{n,t}-e_{n,t}, b^{\max}\}.
	\end{aligned}    
\end{equation}
The energy consumption at each time slot cannot exceed the current remaining energy in the battery. Thus, the energy causality constraint is given by
\begin{equation}\label{causality_energy}
	\begin{aligned}
		0\leq e_{n,t}\leq b_{n,t}, \forall n,t.
	\end{aligned}    
\end{equation}

\section{Convergence Analysis and Problem Formulation}

In this section, we first present convergence analysis for wireless FL with EH devices, deriving an upper bound for the optimality gap between the expected and optimal global loss. Then, we formulate a joint device scheduling and power control problem to minimize the expected optimality gap.

\subsection{Convergence Analysis}
To facilitate the convergence analysis, we make the following standard assumptions, which are widely adopted in the literature on FL \cite{sery2021over,nguyen2021gradual,wang2024communication,zhang2023online,gafni2024federated}.

\begin{assump}\label{smoothness_assump}
	(Convexity and Smoothness) The local loss function $F_n(\v{w})$ is convex and $L$-smooth. For any $\v{a}, \v{b} \in \b{R}^d$, $n \in [N]$, we have
	\begin{align}\label{smoothness}
		\begin{split}
			F_n (\v{a}) - F_n (\v{b}) \leq \nabla F_n (\v{b})^{\s{T}}(\v{a}-\v{b})+\frac{L}{2} \|\v{a}-\v{b}\|^2.
		\end{split}
	\end{align}
\end{assump}



\begin{assump}\label{ass_bounded_var_assump}
	(Bounded Variance)
	The stochastic gradient of each device is unbiased and has a bounded local variance, i.e., for any $\v{w}\in \b{R}^d $, $i\in [m]$, we have
	\begin{equation}\label{bounded_grad_eq}
		\begin{aligned}
			\mathbb{E}_{\ca{\xi} \sim \ca{D}_i}\left[ \left\| \nabla f_i(\v{w}, \ca{\xi}) - \nabla F_i(\v{w}) \right\|^2 \right] \leq \sigma_l^2.
		\end{aligned}    
	\end{equation}
	The local loss functions have a dissimilarity bound that characterizes their heterogeneity. For any $\v{w}\in \b{R}^d $, $i\in [m]$, we have
	\begin{equation}\label{bounded_heto_eq}
		\begin{aligned}
			\left\|\nabla F_i(\v{w}) - \nabla F(\v{w}) \right\|^2 \leq \sigma_g^2.
		\end{aligned}    
	\end{equation}
\end{assump}

\begin{assump}\label{ass_bounded_G_assump}
	(Bounded Stochastic Gradient)
	The expected squared norm of the stochastic gradients is uniformly bounded, i.e., for any $\v{w}\in \b{R}^d $, $i\in [m]$, we have
	\begin{equation}\label{bounded_G}
		\begin{aligned}
			\mathbb{E}\left[ \left\| \nabla f_i(\v{w}, \ca{\xi}) \right\|^2 \right] \leq G^2.
		\end{aligned}    
	\end{equation}
\end{assump}

Based on the assumptions above, we characterize the convergence bound of wireless FL with EH for general convex loss functions.
Specifically, we analyze the convergence of the average model, which is defined as
\begin{equation}\label{w_bar}
	\begin{aligned}
		\v{\bar w}_t^{(k)} = \frac{1}{N}\sum_{n=1}^{N}  \v{w}_{n,t}^{(k)}.
	\end{aligned}    
\end{equation}

\begin{thm}\label{thm_convergence_bound}
	Given the initial model $\v{w}_1$, for any device scheduling strategy $\{\beta_{1,t},\ldots,\beta_{N,t}\}_{t=1}^T$ and power control strategy $\{p_{1,t},\ldots,p_{N,t}\}_{t=1}^T$, the optimality gap between the expected and optimal global loss is upper bounded by
	\begin{equation}
		\begin{aligned}\label{convergence_bound}
			&\mathbb{E}\Biggl[\frac{1}{K T}\sum_{t=1}^{T}\sum_{k=1}^{K}
			\bigl(F(\bar{\v{w}}_t^{(k)}) - F(\v{w}^\star)\bigr)\Biggr]\\
			\leq &
			\frac{2 L \|  \v{w}_1 - \v{w}^* \|^2}{\sqrt{K T }}
			+ \frac{\sigma_l^2}{4L N \sqrt{K T}}
			+ \frac{\sigma_l^2 }{4L T}
			+ \frac{9 \sigma_g^2 K }{8L T} \\
			& +\frac{ 2 L G^2}{ N \sqrt{KT} }\sum_{t=1}^{T}\sum_{n=1}^{N}\left( 1 - \beta_{n,t} ( 1 -q_{n,t}(h_{n,t},p_{n,t})) \right),
		\end{aligned}    
	\end{equation}
	where the learning rate satisfies $\eta = \frac{1}{4 L \sqrt{K T}}$.
\end{thm}
\begin{proof}
	{The proof is presented in Appendix A.}
\end{proof}

\begin{rem}
	%
	
	The first four terms on the right-hand side (RHS) of \eqref{convergence_bound} arise from the initial error (i.e., the gap between the initial model ${\v{w}}_{1}$ and the optimal model $\v{w}^{*}$) as well as other learning-related factors (e.g., device heterogeneity and data heterogeneity).
	These terms gradually diminish as the number of time slots $T$ increases.
	In contrast, the final term represents the persistent error introduced by partial device participation and transmission packet drops, and this error does not vanish as $T$ grows.
	In particular, the convergence rate of wireless FL can be improved by scheduling more devices and increasing their transmission power.
	However, in practical scenarios, both device scheduling and transmission power are constrained by the battery energy level.
	Thus, the transmission policy should be carefully designed to utilize the limited and uncertain harvested energy effectively.
\end{rem}

\subsection{Problem Formulation}
In this paper, we aim to optimize the device scheduling strategy $\{\beta_{n,t}\}$ and power control strategy $\{p_{n,t}\}$ to minimize the expected global loss after $T$ time slots. However, directly solving this problem is challenging since the global loss is hard to express in a closed form in terms of the transmission strategy.
To address this issue, we approximate the actual global loss with its convergence bound derived in Theorem \ref{thm_convergence_bound}, and formulate an optimality gap minimization problem.
In particular, we ignore the first four terms on the RHS of \eqref{convergence_bound}, as they are independent of the device scheduling and transmission power and go to zero as $T$ increases.
Hence, the resulting optimization problem is formulated as follows:
\begin{align}\label{P1_orig}
	\underset{\{\v{\beta}_{t},\v{p}_t\}_{t=1}^{T}}{\text{min}} &\;\sum_{t=1}^{T}\sum_{n=1}^{N}\frac{ 2 L G^2 \left( 1 - \beta_{n,t} ( 1 -q_{n,t}(h_{n,t},p_{n,t})) \right)}{ N \sqrt{KT}  } \nonumber \\
	\text{s.t.} \quad\; &~ \text{C1:\;}\sum_{n=1}^{N}\beta_{n,t}\leq R, \forall t, \nonumber \\
	&~\text{C2:\;\;} 0\leq e_{n,t}\leq b_{n,t}, \forall n,t,
\end{align}
where $\v{\beta}_{t}=\{\beta_{1,t},\ldots,\beta_{N,t}\}$ and $\v{p}_{t}=\{p_{1,t},\ldots,p_{N,t}\}$.
Constraint C1 corresponds to the bandwidth constraint, while constraint C2 is the energy causality constraint.
We remark that instead of adopting heuristic performance metrics (e.g., energy consumption and the number of scheduled devices), we provide a more reasonable objective function based on the convergence bound derived in Theorem \ref{thm_convergence_bound}.

We can further simplify problem \eqref{P1_orig} by substituting the device scheduling indicator $\beta_{n,t}$ with the transmission power $p_{n,t}$. The device $n$ is scheduled only when its transmit power $p_{n,t}$ is positive, i.e., $\beta_{n,t}=\mathbbm{1}_{\left\lbrace p_{n,t}> 0\right\rbrace }$.
Thus, problem \eqref{P1_orig} can be simplified as
\begin{equation}\label{P1_reformulated}
	\begin{aligned}
		\underset{\{\v{p}_t\}_{t=1}^{T}}{\text{min}} &~\sum_{t=1}^{T}\sum_{n=1}^{N}\frac{ 2 L G^2 q_{n,t}(h_{n,t},p_{n,t}) }{N \sqrt{KT}  } \\
		\text{s.t.} ~~ &~\;\text{C2},\\
		&~\;\text{C3: }\sum_{n=1}^{N} \mathbbm{1}_{\left\lbrace p_{n,t}> 0\right\rbrace } \leq R, \forall t.
	\end{aligned}
\end{equation}
Constraint C3 is an equivalent transformation of constraint C1.

The reformulated problem \eqref{P1_reformulated} falls into the category of stochastic optimization problems, and it is hard to solve due to the inherent uncertainty of channel conditions and harvested energy.
We note that the battery energy dynamics naturally exhibit the Markov property, and the first-order Markov chain sufficiently models slow fading channels. Motivated by these observations, we model problem \eqref{P1_reformulated} as a finite-horizon MDP and derive its optimal policy in Section \ref{Sec_MDP}.

\section{MDP Formulation and Optimal Policy}\label{Sec_MDP}
In this section, we first propose an MDP framework for deriving the optimal transmission policy.
We prove that the optimal policy can be computed offline using the Bellman equation.
Then, we demonstrate that the optimal policy exhibits a monotone structure with respect to the battery and channel states.

\subsection{MDP Formulation}
To capture the temporal correlation of practical wireless channels, we adopt a finite-state Markov channel model, which has been extensively used to approximate various mathematical and experimental fading models (e.g., indoor channels, Rayleigh fading channels, and Rician fading channels) \cite{sadeghi2008finite}. The channel is quantified into $N_h$ states, with its corresponding channel gains quantified by $\ca{H}=\left\lbrace H_{1},H_{2},\ldots,H_{N_h} \right\rbrace$.
Without loss of generality, we assume that $H_i < H_{i^\prime}$, for $i < i^\prime$. The transition probability matrix of the $n$-th device's channel is denoted by
\begin{equation}\label{def_phi}
	\begin{aligned}
		\Psi_n=\left[\begin{array}{lll}
			\psi_{n}^{(1,1)} & \cdots & \psi_{n}^{(1,N_h)} \\
			\vdots & \ddots & \vdots \\
			\psi_{n}^{(N_h,1)} & \cdots & \psi_{n}^{(N_h,N_h)}
		\end{array}\right]
	\end{aligned}    
\end{equation}
with its $(i,j)$-th entry $\psi_{n}^{(i,j)}=\b{P}(h_{n,t+1}=H_{j}|h_{n,t}=H_{i})$. Moreover, the harvested energy $u_{n,t}$ is assumed to be i.i.d. across time slots according to some probability distribution $\b{P}(U_n)$, which depends on the external environment of device $n$.
The battery state of each device is also quantized into finite states, denoted by $\ca{B}$.
Modeling the wireless channels and battery energy dynamics as homogeneous finite-state Markov chains, we formulate problem \eqref{P1_reformulated} as a finite-horizon MDP.
The MDP is defined by the tuple $\left(\ca{S}, \ca{P}, \b{P}(\v{h}_{t+1},\v{b}_{t+1}|\v{h}_{t},\v{b}_{t},\v{p}_t), c(\v{h}_t,\v{p}_t)\right) $, each component of which is described as follows:
\begin{itemize}
	\item \textbf{States}: Denote $s_{n,t}=(h_{n,t},b_{n,t})$ as the state associated with device $n$ at time slot $t$.
	The state space $\ca{S}_n$ of device $n$ is a composite space with $\ca{S}_n=\ca{H}\times\ca{B}$, where $\times$ denotes the Cartesian product. The state of the entire system at time slot $t$ is expressed as $s_{t} = (s_{1,t},\ldots,s_{N,t})\in \ca{S}$, where $\ca{S}=\ca{S}_1\times\ldots\times\ca{S}_N = \ca{H}^N\times\ca{B}^N $ is the state space of the entire system.
	\item \textbf{Actions}: In practice, wireless edge devices can be programmed to have a finite set of transmission power levels.
	Thus, we assume each device $n$ discretely selects its transmission power $p_{n,t}\in \ca{P}_n$, where $\ca{P}_n$ is the finite action space of device $n$. If $p_{n,t}=0$, device $n$ remains idle at time slot $t$.
	The action space of the whole system is given by $\ca{P}=\ca{P}_1\times\ldots\times\ca{P}_N$.
	
	\item \textbf{Transition probabilities}: The transition probability $\b{P}(\v{h}_{t+1},\v{b}_{t+1}|\v{h}_{t},\v{b}_{t},\v{p}_t)$ maps a state-action pair at time slot $t$ onto a distribution of states at the next time slot $t+1$. Specifically, the transition probability can be factorized as
	\begin{equation}\label{trans_prob}
		\begin{aligned}
			&\b{P}(\v{h}_{t+1},\v{b}_{t+1}|\v{h}_{t},\v{b}_{t},\v{p}_t)\\
			=&\prod_{n=1}^{N}\b{P}(h_{n,t+1},b_{n,t+1}|h_{n,t},b_{n,t},p_{n,t})\\
			=&\prod_{n=1}^{N}\b{P}(h_{n,t+1}|h_{n,t})\b{P}(b_{n,t+1}|h_{n,t},b_{n,t},p_{n,t}).
		\end{aligned}    
	\end{equation}
	Considering the battery dynamics shown in \eqref{evo_energy},  the first-order conditional probability of the process $b_{n,t}$ is given by
	\begin{equation}\label{trans_prob_b}
		\begin{aligned}
			\b{P}(b_{n,t+1}|h_{n,t},b_{n,t},p_{n,t})\quad\quad\quad\quad\quad\quad\quad\quad\quad\quad\quad\quad~~&\\
			=\begin{cases}\b{P}\left(U_n = b_{n,t+1}- b_{n,t} + e_{n,t} \right) ,&\text{if  } b_{n,t+1} < b^{\max};
				\\ \b{P}\left(U_n \geq b^{\max}- b_{n,t} + e_{n,t} \right),&\text{if  } b_{n,t+1} = b^{\max}.\end{cases}
		\end{aligned}    
	\end{equation}
	\item \textbf{One-Step Cost}: Inspired by the objective function in problem \eqref{P1_reformulated}, we consider a cost function that quantifies the penalty associated with packet drops and partial device participation. Specifically, the one-step cost at time slot $t$ is defined as 
	\begin{equation}
		\begin{aligned}\label{cost}
			c(\v{h}_{t},\v{p}_{t})=&\sum_{n=1}^{N}c_n({h}_{n,t},{p}_{n,t}),\\
		\end{aligned}
	\end{equation}
	where $c_n({h}_{n,t},{p}_{n,t}) = \frac{ 2 L G^2 q_{n,t}(h_{n,t},p_{n,t}) }{N \sqrt{KT}  }$ is the one-step cost of device $n$.
\end{itemize}

The objective of the MDP is to derive the optimal policy that minimizes the expected cumulative cost.
The policy $\pi=\{\pi_t : \ca{S} \rightarrow \ca{P}\}_{t=1}^T$ determines the transmission power $\v{p}_t$ based on the state $(\v{h}_t,\v{b}_t)$ at each time slot $t$. For a given policy $\pi$, we define the expected cumulative cost as
\begin{equation}\label{def_J}
	\begin{aligned}
		J(\pi)=\b{E}_{\pi}\left[ \sum_{t=1}^{T} c(\v{h}_{t},\v{p}_{t}) | \v{h}_{1},\v{b}_{1}\right],\\
	\end{aligned}    
\end{equation}
where $(\v{h}_{1},\v{b}_{1})$ is the initial state. By defining the expected cumulative cost in this way, problem \eqref{P1_reformulated} is reformulated as the following problem that computes the optimal policy $\pi^*$ to minimize the expected optimality gap.

\begin{problem}\label{Problem1}
	\textup{(Optimality Gap Minimization Problem)}
	\begin{equation}
		\begin{aligned}
			\pi^*=\underset{\pi}{\textup{argmin}}~&J(\pi)\\
			\textup{s.t.} \quad & \textup{C2 and C3}.
		\end{aligned}    
	\end{equation}
\end{problem}

With Problem \ref{Problem1} formally established as a finite-horizon MDP, we next characterize the optimal transmission policy by leveraging dynamic programming. Specifically, in the following subsection, we derive the optimal transmission policy via the Bellman equation and analyze its structural properties.


%
%
%
%

\subsection{Optimal Policy}
In this subsection, we show that the optimal policy can be derived offline according to the Bellman equation. Moreover, we prove that the optimal policy possesses a monotone structure with respect to the battery and channel states.

We first derive the optimal policy for Problem \ref{Problem1} in the following theorem.

\begin{thm}\label{bellman}
	Given the initial state $(\v{h}_1, \v{b}_1)$, the optimal value $V_1(\v{h}_1,\v{b}_1)$ of the objective in Problem \ref{Problem1} can be computed recursively according to the Bellman equation. Specifically, for $1\leq t \leq T-1$, we have
	\begin{equation}\label{bellman_equ}
		\begin{aligned}
			&V_{t}(\v{h}_t,\v{b}_t)\\
			=&\min_{\v{p}_t\in \ca{P}_{t}}\left\{c(\v{h}_t,\v{p}_t)+\b{E}\left[V_{t+1}(\v{h}_{t+1},\v{b}_{t+1})|\v{h}_t,\v{b}_t,\v{p}_t\right]\right\},\\
		\end{aligned}    
	\end{equation}
	and the terminal condition is given by
	\begin{equation}\label{bellman_equ_terminal}
		\begin{aligned}
			V_{T}(\v{h}_T,\v{b}_T)=&\min_{\v{p}_T\in \ca{P}_{T}}c(\v{h}_T,\v{p}_T),\\
		\end{aligned}    
	\end{equation}
	where $\ca{P}_{t}=\left\lbrace \v{p}_{t}  \in \ca{P}\!:\! \sum_{n} \!\mathbbm{1}_{\left\lbrace p_{n,t}> 0\right\rbrace } \!\leq \! R,0\leq e_{n,t}\leq b_{n,t}, \forall n\right\rbrace$
	denotes the feasible set of transmission power actions at time slot $t$.
\end{thm}
\begin{proof}
	{The proof is presented in Appendix B.}
\end{proof}

After solving the Bellman equation through backward induction, the optimal transmission power $\v{p}_t^*$ can be derived as follows: $\forall 1\leq t \leq T-1,$
\begin{equation}\label{P1_opt_solution}
	\begin{aligned}
		\v{p}^*_t\!=\!&\argmin_{\v{p}_t\in \ca{P}_{t} }\!\left\{c(\v{h}_t,\v{p}_t)\!+\!\b{E}\!\left[V_{t+1}\!(\v{h}_{t+1},\v{b}_{t+1})|\v{h}_t,\v{b}_t,\v{p}_t\right]\right\}\!\!,\\
	\end{aligned}
\end{equation}
and
\begin{equation}\label{P1_opt_solution_terminal}
	\begin{aligned}
		\v{p}^*_T=\argmin_{\v{p}_T\in \ca{P}_{T}} \; c(\v{h}_T,\v{p}_T).\\
	\end{aligned} 
\end{equation}
We remark that the optimal policy for Problem \ref{Problem1} can only be obtained numerically according to the Bellman equation, and there is no closed form solution.

Next, we give some insights into the optimal policy by studying its monotone structure. To facilitate analysis, we first prove the convexity of the value function $V_{t}(\v{h}_t,\v{b}_t)$ with respect to the battery state in the following lemma.
\begin{lem}\label{V_convex}
	Given any fixed $\v{h}_t$, the value function $V_{t}(\v{h}_t,\v{b}_t)$ is convex in $\v{b}_t$ for $1 \leq t \leq T$.
\end{lem}
\begin{proof}
	{The proof is presented in Appendix C.}
\end{proof}

By applying Lemma \ref{V_convex}, we further show that the optimal policy possesses a monotone structure with respect to the battery state in the following theorem.
\begin{thm}\label{p_increasing_b}
	Given any fixed $\v{h}_t$ and $\left\lbrace b_{m,t},m\in [N]\backslash n\right\rbrace $, the optimal transmission power $p_{n,t}^*$ of device $n$ at time slot $t$ is non-decreasing in its current battery level $b_{n,t}$ for $1 \leq t \leq T$.
\end{thm}
\begin{proof}
	It is obvious that at the final time slot $T$, $p_{n,T}^*$ is non-decreasing in $b_{n,T}$, since all the remaining energy should be used up in the final time slot. Then, we prove that $p_{n,t}^*$ is non-decreasing in $b_{n,t}$ for $ 1 \leq t\leq T-1$. We first define the RHS of \eqref{P1_opt_solution} with fixed $\v{h}_t$ and $\left\lbrace b_{m,t},m\in [N]\backslash n\right\rbrace $ as follows:
	\begin{equation}\label{sub_fucntion_z}
		\begin{aligned}
			&Z_{n,t}(b_{n,t},p_{n,t})\\
			&=c(\v{h}_t,\v{p}_t)+\b{E}\left[V_{t+1}(\v{h}_{t+1},\v{b}_{t+1})|\v{h}_t,\v{b}_t,\v{p}_t\right].\\
		\end{aligned} 
	\end{equation}
	Then, we introduce a definition, namely, submodularity. A function $f:\ca{X}\times\ca{Y}\rightarrow\ca{R}$ is a submodular function if it satisfies the following condition:
	\begin{equation}\label{submodular_equ}
		\begin{aligned}
			f(x^{+},y^{+})+ f(x^{-},y^{-})\leq  f(x^{+},y^{-})+f(x^{-},y^{+}),
		\end{aligned}    
	\end{equation}
	where $x^{+}\geq x^{-}$ and $y^{+}\geq y^{-}$.
	
	To show that $p_{n,t}^*$ is non-decreasing in $b_{n,t}$, it is sufficient to show that $Z_{n,t}(b_{n,t},p_{n,t})$ is submodular in $(b_{n,t},p_{n,t})$ \cite{puterman2014markov}. Here we provide some intuitive explanations. According to the submodularity of $Z_{n,t}(b_{n,t},p_{n,t})$, we have that if $Z_{n,t}(b_{n,t}^{-},p_{n,t}^{+})\leq Z_{n,t}(b_{n,t}^{-},p_{n,t}^{-})$, then $ Z_{n,t}(b_{n,t}^{+},p_{n,t}^{+})\leq Z_{n,t}(b_{n,t}^{+},p_{n,t}^{-})$, where $b_{n,t}^{+}\geq b_{n,t}^{-}$ and $p_{n,t}^{+}\geq p_{n,t}^{-}$. This implies that if using $p_{n,t}^{+}$ leads to a lower cost than using $p_{n,t}^{-}$ when the battery level is $b_{n,t}^{-}$, then it is the same when the battery level is $b_{n,t}^{+}$, which indicates the transmission power $p_{n,t}$ is non-decreasing in $b_{n,t}$.
	
	It is evident that the first term of $Z_{n,t}(b_{n,t},p_{n,t})$, i.e., $c(\v{h}_t,\v{p}_t)$, is submodular in $(b_{n,t},p_{n,t})$ since it is independent of $b_{n,t}$.
	Then, define $\v{x}_t = \v{b}_{t}-\v{e}_{t}$ with its $n$-th entry $x_{n,t}={b}_{n,t}-{e}_{n,t}$. Given fixed $\left\lbrace b_{m,t},m\in [N]\backslash n\right\rbrace $, we denote the second term of $Z_{n,t}(b_{n,t},p_{n,t})$ as
	\begin{equation}\label{def_K}
		\begin{aligned}
			K_{n,t}(x_{n,t})=\b{E}\left[V_{t+1}(\v{h}_{t+1},\min\{\v{x}_t+\v{u}_{t},\v{b}^{\max}\}|\v{h}_t,\v{b}_t,\v{p}_t\right],
		\end{aligned}
	\end{equation}
	where $x_{m,t}$ is fixed for $ m\in [N]\backslash n$. According to the convexity of $V_{t}(\v{h}_t,\v{b}_t)$ in $\v{b}_t$, as shown in Lemma \ref{V_convex}, we have
	\begin{equation}\label{K_cvx}
		\begin{aligned}
			K_{n,t}(x_{n,t}+\epsilon) - K_{n,t}(x_{n,t}) \leq  K_{n,t}(x_{n,t}^{\prime}+\epsilon) - K_{n,t}(x_{n,t}^{\prime}),\\
			\forall x_{n,t}\leq x_{n,t}^{\prime}, \epsilon\geq 0.
		\end{aligned}
	\end{equation}
	Now let $x_{n,t}=b_{n,t}^{-}-e_{n,t}^{+}$, $x_{n,t}^{\prime}=b_{n,t}^{-}-e_{n,t}^{-}$, and $\epsilon=b_{n,t}^{+}-b_{n,t}^{-}$, where $e_{n,t}^{+}$ and $e_{n,t}^{-}$ correspond to the energy consumption under the transmisson power $p_{n,t}^{+}$ and $ p_{n,t}^{-}$, respectively. Then, we can obtain that the second term of $Z_{n,t}(b_{n,t},p_{n,t})$ is also submodular in $(b_{n,t},p_{n,t})$. Thus, $p_{n,t}^*$ is non-decreasing in $b_{n,t}$.
\end{proof}

Next, we demonstrate that the optimal policy also exhibits a monotone structure with respect to the channel state. To this end, we first make the following assumption. 
\begin{assump}\label{channel_transmission_matrix_assumption}
	(The stochastic monotonicity of the channel transition probability matrix) The channel transition probability matrix $\Psi_n$, $\forall n \in [N]$ is stochastic monotone. For $1 \leq i\leq i^{\prime}\leq N_h$, $1\leq k\leq N_h$, we have
	\begin{align}\label{channel_transmission_matrix_equ}
		\begin{split}
			\sum_{j=k}^{N_h}\psi_{n}^{(i^{\prime},j)} \geq  \sum_{j=k}^{N_h}\psi_{n}^{(i,j)}.
		\end{split}
	\end{align}
\end{assump}

Assumption \ref{channel_transmission_matrix_assumption} is widely used in the Markov modeling of various fading channels, such as Rayleigh and Rician fading channels \cite{sadeghi2008finite}. It allows for an intuitive interpretation: better channel quality at the current time slot indicates a higher probability of maintaining better channel quality at the next time slot. Based on Assumption \ref{channel_transmission_matrix_assumption}, we prove that the optimal policy possesses a monotone structure with respect to the channel state in the following theorem.

\begin{thm}\label{p_decreasing_h}
	Given any fixed $\v{b}_t$ and $\left\lbrace h_{m,t},m\in [N]\backslash n\right\rbrace $, if Assumption \ref{channel_transmission_matrix_assumption} holds, the optimal transmission power $p_{n,t}^*$ of device $n$ at time slot $t$ is non-increasing in its current channel state $h_{n,t}$ for $1 \leq t \leq T$.
\end{thm}
\begin{proof}
	{The proof is presented in Appendix D.}
\end{proof}

\begin{figure}[tbp]
	\renewcommand\figurename{\small Fig.}
	\centering \vspace*{3pt} \setlength{\baselineskip}{10pt}
	\includegraphics[width = 3.15in,center]{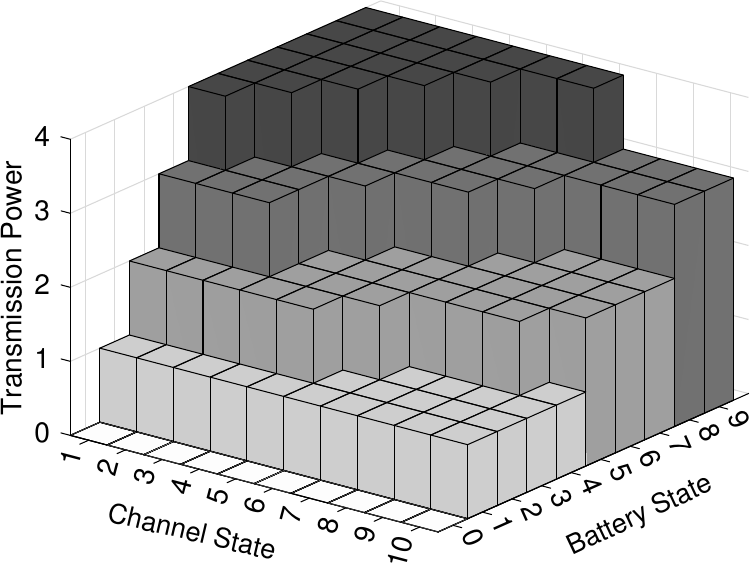}
	\vspace{-8pt}
	\caption{An illustration for the monotone structure in the optimal policy of one typical device.}\label{threshold_fig}
	\vspace{-3pt}
\end{figure}

Fig. \ref{threshold_fig} demonstrates the monotone structure for the optimal policy of one typical device.
The benefits for the monotone structure are twofold. First, it reduces the storage space for online implementation, as only the boundary states used for switching transmission power need to be stored. Second, the monotone structure can reduce the computational overhead for solving the Bellman equation.
According to the monotone structure, the standard value iteration algorithm can be modified as follows.
At time slot $t$, within the defined state space $[\v{h}_{\textup{min}},\v{h}_{\textup{max}}]\times [\v{b}_{\textup{min}},\v{b}_{\textup{max}}]$, we first compute the optimal actions at the corners, i.e., $\v{p}_{t}^*(\v{h}_{\textup{min}},\v{b}_{\textup{max}})$ and $\v{p}_{t}^*(\v{h}_{\textup{max}},\v{b}_{\textup{min}})$, according to the Bellman equation. Then, we check whether there exists an action $\v{p}^\#_t$ such that
\begin{equation}\label{modified_Bellman_equation}
	\begin{aligned}
		\v{p}^{\#}_t = \v{p}_{t}^*(\v{h}_{\textup{min}},\v{b}_{\textup{max}}) = \v{p}_{t}^*(\v{h}_{\textup{max}},\v{b}_{\textup{min}}).
	\end{aligned}
\end{equation}
If such an action $\v{p}^\#_t$ exists, it can be determined as the optimal action for the entire region, and we update the value function for this region as follows:
\begin{equation}\label{bellman_equ_structure}
	\begin{aligned}
		V_{t}(\v{h}_t,\v{b}_t)=c(\v{h}_t,\v{p}^\#_t)+\b{E}\!\left[V_{t+1}(\v{h}_{t+1},\v{b}_{t+1})|\v{h}_t,\v{b}_t,\v{p}^\#_t\right],\\
		\forall \v{h}_t \in [\v{h}_{\textup{min}},\v{h}_{\textup{max}}], \v{b}_t\in [\v{b}_{\textup{min}},\v{b}_{\textup{max}}].
	\end{aligned}    
\end{equation}
If such a $\v{p}^\#_t$ does not exist, we calculate the midpoints $\v{h}_{\textup{mid}} = \floor{\frac{\v{h}_{\textup{min}}+\v{h}_{\textup{max}}}{2}}$ and $\v{b}_{\textup{mid}} = \floor{\frac{\v{b}_{\textup{min}}+\v{b}_{\textup{max}}}{2}}$, and we split the current region into four quadrants: $[\v{h}_{\textup{min}},\v{h}_{\textup{mid}}]\times [\v{b}_{\textup{min}},\v{b}_{\textup{mid}}]$, $[\v{h}_{\textup{min}},\v{h}_{\textup{mid}}]\times [\v{b}_{\textup{mid}},\v{b}_{\textup{max}}]$, $[\v{h}_{\textup{mid}},\v{h}_{\textup{max}}]\times [\v{b}_{\textup{min}},\v{b}_{\textup{mid}}]$, and $[\v{h}_{\textup{mid}},\v{h}_{\textup{max}}]\times [\v{b}_{\textup{mid}},\v{b}_{\textup{max}}]$. Then, we apply the same checking and subdivision process recursively to each of these subregions. The recursion terminates when the optimal actions for all states are determined. We note that equation \eqref{bellman_equ_structure} avoids the brute-force search over the action space $\ca{P}$, thereby significantly reducing the computational complexity. According to Theorems \ref{p_increasing_b} and \ref{p_decreasing_h}, the modified algorithm converges to the optimal policy as the standard value iteration algorithm does.

Although the monotone structure can reduce the computational complexity, computing the optimal transmission policy of multiple devices still remains challenging.
Specifically, the computational complexity for deriving the optimal policy is $\ca{O}(|\ca{S}|^2|\ca{P}|)$ \cite{balaji2018complexity}, where both the state space $\ca{S}$ and the action space $\ca{P}$ grow exponentially with the number of devices $N$. Thus, the computational complexity also increases exponentially in $N$. To address the curse of dimensionality, in Section \ref{asy_opt}, we trade off optimality for complexity and develop an asymptotically optimal policy, the computational complexity of which increases only linearly with $N$.

\section{Asymptotically Optimal Policy: Relax-and-Truncate Approach}\label{asy_opt}

In this section, we first develop a low-complexity algorithm, termed the relax-and-truncate approach, whose computational complexity increases only linearly with the number of devices, $N$. Then, we prove that the proposed algorithm is asymptotically optimal as $N$ increases.

\subsection{Device Level Decomposition}

We start by observing that Problem \ref{Problem1} is intractable due to constraint C3, which couples the transmission power actions $\left\lbrace p_{n,t},n\in [N]\right\rbrace $, resulting in the exponential complexity.
To address this issue, we relax the hard constraint C3 into a time-average constraint. First, we define the time-average number of scheduled devices under a policy $\pi$ as
\begin{equation}\label{def_R_lambda}
	\begin{aligned}
		R(\pi)=&\b{E}_{\pi}\left[ \frac{1}{T}\sum_{t=1}^{T}\sum_{n=1}^{N} \mathbbm{1}_{\left\lbrace p_{n,t}> 0\right\rbrace }\right].\\
	\end{aligned}    
\end{equation}
Then, we formulate the relaxed problem as follows.
\begin{problem}\label{Problem2}
	\textup{(Relaxed Optimality Gap Minimization Problem)}
	\begin{equation}\label{P2}
		\begin{aligned}
			\overline \pi^*=\underset{\pi}{\textup{argmin}}~~&J(\pi)\\
			\textup{s.t.} ~~~~ &\textup{C2},\\
			&\overline{\textup{C3}}:R(\pi) \leq R.
		\end{aligned}    
	\end{equation}
\end{problem}
The relaxed bandwidth constraint $\overline{\textup{C3}}$ allows more than $R$ devices to upload local gradients at certain time slots, provided that the time-average bandwidth constraint is satisfied over all time slots.
The optimal policy for Problem \ref{Problem2} is referred to as the optimal relaxed policy $\overline \pi^*$ hereinafter.

The relaxed constraint $\overline{\textup{C3}}$ in Problem \ref{Problem2} still couples the power actions of all devices. To tackle this issue, we adopt a device level decomposition by using the Lagrange approach \cite{altman2021constrained}.
Specifically, we introduce a Lagrange multiplier $\lambda\geq0$ and define the Lagrangian associated with Problem \ref{Problem2} as
\begin{equation}\label{def_lagrangian}
	\begin{aligned}
		&\ca{L}(\pi,\lambda)\\
		=&\b{E}_{\pi}\left[ \sum_{t=1}^{T} \left( c(\v{h}_{t},\v{p}_{t})  +\lambda\left(\sum_{n=1}^{N}\mathbbm{1}_{\left\lbrace p_{n,t}> 0\right\rbrace }-R\right)\right) \right].\\
	\end{aligned}    
\end{equation}
The Lagrange multiplier $\lambda$ can be interpreted as the cost of scheduling one device at a time slot. For a given $\lambda$, we define the Lagrange dual function as
\begin{equation}\label{def_lagrangian_dual}
	\begin{aligned}
		\ca{L}^*(\lambda)=\min_{\pi}~&\ca{L}(\pi,\lambda)\\
		\text{s.t.} ~~ & \textup{C2}.
	\end{aligned}    
\end{equation}
A policy that achieves $\ca{L}^*(\lambda)$ is called the $\lambda$-optimal policy $\overline \pi_{\lambda}$.

Note that the dimension of the state space $\ca{S}$ is finite and the one-step cost is bounded below, i.e., $c(\v{h},\v{p})\geq 0,\forall \v{h},\v{p}$. Under these conditions, there exists an optimal Lagrange multiplier $\lambda^*$ such that $J(\overline \pi^*)=\ca{L}^*(\lambda^*)$ and $\lambda^*\left( R(\overline \pi_{\lambda^*})-R\right)=0$ \cite[Theorem 12.8]{altman2021constrained}.
Thus, the optimal relaxed policy $\overline \pi^*$ can be computed by using a two-stage iterative algorithm: 1) For a given $\lambda$, we find the $\lambda$-optimal policy; 2) we update $\lambda$ by applying a bisection method. These two steps are detailed in the following.

\subsubsection{Optimal Policy for a Fixed Lagrange Multiplier}
Given a fixed $\lambda$, the problem of finding the $\lambda$-optimal policy $\overline \pi_{\lambda}$ is separable across devices and can be decoupled into $N$ per-device subproblems. Specifically, the Lagrangian in \eqref{def_lagrangian} can be expressed equivalently as $\ca{L}(\pi,\lambda)=\sum_{n=1}^{N}\ca{L}_n(\pi_n,\lambda)$, where
\begin{equation}\label{def_lagrangian_n}
	\begin{aligned}
		&\ca{L}_n(\pi_n,\lambda)\\
		=&\b{E}_{\pi_n \!\!}\left[ \sum_{t=1}^{T}\left( c_n({h}_{n,t},{p}_{n,t})\!+\!\lambda \mathbbm{1}_{\left\lbrace p_{n,t}> 0\right\rbrace }\!-\!\frac{\lambda R}{N} \right) \!\right]\!,
	\end{aligned}    
\end{equation}
and $\pi_n=\{\pi_{n,t}:\ca{S}_n \rightarrow \ca{P}_n\}_{t=1}^T$ is the policy of device $n$. Then, the problem of finding $\overline \pi_{\lambda}$ reduces to finding $N$ per-device $\lambda$-optimal policies $\overline \pi_{n,\lambda}$ as follows.

\begin{problem}\label{Problem3}
	(Decoupled Optimality Gap Minimization)
	\begin{equation}\label{P3}
		\begin{aligned}
			\overline \pi_{n,\lambda}=\underset{\pi_n}{\textup{argmin}}~~&\ca{L}_n(\pi_n,\lambda)\\
			\textup{s.t.} ~~~~ &\textup{C2}.
		\end{aligned}    
	\end{equation}
\end{problem}
The decoupled Problem \ref{Problem3} for each device $n$ can be modeled as an MDP and solved by value iteration or policy iteration algorithms.

\subsubsection{Determination of the Optimal Lagrange Multiplier}

According to \cite[Lemma 3.1]{beutler1985optimal}, $R(\overline \pi_{\lambda})$ decreases as $\lambda$ increases. Intuitively, the Lagrange multiplier $\lambda$ can be interpreted as the cost of scheduling one device; thus, increasing $\lambda$ imposes a greater penalty on the scheduling cost in the Lagrangian.
Our goal is to identify the smallest value of $\lambda$ that satisfies the relaxed bandwidth constraint $\overline{\textup{C3}}$.
%
Formally, the optimal Lagrange multiplier is defined as
\begin{equation}\label{opt_lamda}
	\begin{aligned}
		\lambda^*=\inf \left\lbrace\lambda\geq 0:R(\overline \pi_{\lambda}) \leq R \right\rbrace. 
	\end{aligned}  
\end{equation} 
To search for the optimal Lagrange multiplier $\lambda^*$, we apply the bisection method, which exploits the monotonicity of $R(\overline \pi_{\lambda})$ with respect to $\lambda$. We initialize $\lambda^-=0$ and $\lambda^+$ as a sufficiently large number such that $R(\overline \pi_{\lambda^+})< R$. The algorithm terminates when $\lambda^+ - \lambda^-<\epsilon$, where $\epsilon$ is a sufficiently small constant.
We summarize the complete procedure for deriving the optimal relaxed policy $\overline \pi^*$ in Algorithm \ref{algorithm1}.

\begin{rem}
	By employing the Lagrange approach, solving the relaxed Problem \ref{Problem2} is reduced to solving $N$ subproblems (i.e. Problem \ref{Problem3}). It is worth noting that, in this manner, the computational complexity of finding the optimal relaxed policy $\overline \pi^*$ increases only linearly with the number of devices $N$, whereas the computational complexity of computing the optimal policy $\pi^*$ grows exponentially with $N$.
\end{rem}

\subsection{Truncation Policy With Hard Bandwidth Constraint}
Note that the hard constraint C3 is not guaranteed to be satisfied under the optimal relaxed policy $\overline \pi^*$. To address this issue, we construct a new truncation policy $\tilde{\pi}$ based on $\overline \pi^*$ to meet the hard constraint C3.
Denote $\overline \Omega_t$ as the set of scheduled devices at time slot $t$ under the optimal relaxed policy $\overline \pi^*$.
The truncation policy $\tilde{\pi}$ at each time slot is carried out in two cases as follows:
\begin{itemize}
	\item If $|\overline \Omega_t|\leq R$, all devices $n\in \overline \Omega_t$ can upload their local gradients. 
	\item Otherwise if $|\overline \Omega_t| > R$, the server schedules a subset of devices $\tilde \Omega_t\subseteq\overline \Omega_t$ \emph{randomly} such that $|\tilde \Omega_t|=R$. Those devices that are in set $\overline \Omega_t$ but not selected in $\tilde \Omega_t$ will not be scheduled to transmit.
	
\end{itemize}

\begin{algorithm}[tbp]
	\caption{Optimal Relaxed Policy via Value Iteration and Bisection Search} \label{algorithm1}
	\textbf{Initialize:} Set $\lambda^-=0$, $\lambda^+$ as a sufficiently large number, and set a small termination criterion $\epsilon>0$\;
	\Repeat{$\lambda^+ - \lambda^-\leq \epsilon$}
	{
		\tcp{Computing the $\lambda$-optimal policy}
		Set $\lambda=\frac{\lambda^-+\lambda^+}{2}$\;
		\For{{\rm \textbf{all devices:}} $n=1,\ldots,N$ {\rm \textbf{in parallel}}} 
		{
			Compute $\overline \pi_{n,\lambda}$ by solving Problem \ref{Problem3} through value iteration algorithm\;
		}
		\tcp{Update the multiplier $\lambda$}
		Compute $R(\overline \pi_{\lambda})$ according to \eqref{def_R_lambda}\;
		\eIf{$R(\overline \pi_{\lambda})\geq R$}{Set $\lambda^-=\lambda$;}{Set $\lambda^+=\lambda$;}
	}
\end{algorithm}

\subsection{Asymptotic Optimality of the Relax-and-Truncate Policy}

In this subsection, we analyze the asymptotic optimality of the proposed relax-and-truncate policy.
We establish an upper bound for the gap between the expected cumulative costs under the truncation policy $\tilde{\pi}$ and the optimal policy $\pi^{*}$ in the following theorem.

\begin{thm}\label{asymptotically_optimal}
	Assume that the proportion of scheduled devices, $\kappa = \frac{R}{N}$, remains constant. The gap between the expected cumulative costs under the truncation policy $\tilde{\pi}$ and the optimal policy $\pi^{*}$ is upper bounded by
	\begin{equation}\label{bound_J_thm}
		\begin{aligned}
			J(\tilde \pi)-J(\pi^*)\leq \frac{ 2 L G^2  \sqrt{T} }{\kappa  \sqrt{N K}  } =\ca{O}\left(\frac{1}{\sqrt{N}}\right).
		\end{aligned}
	\end{equation}
\end{thm}

\begin{proof}
	{The proof is presented in Appendix E.}
\end{proof}

\begin{rem}
	According to Theorem \ref{asymptotically_optimal}, when $\kappa=\frac{R}{N}$ is fixed, the cost gap between the proposed relax-and-truncate policy $\tilde \pi$ and the optimal policy $\pi^*$ is $\ca{O}\left(\frac{1}{\sqrt{N}}\right)$. As $N$ increases, the gap approaches zero. Thus, the proposed relax-and-truncate policy is asymptotically optimal for Problem \ref{Problem1} with hard bandwidth constraint.
\end{rem}

\section{Structure-Enhanced Deep Reinforcement Learning}

In this section, we address a more general scenario in which the statistical knowledge of channels and harvested energy is unknown. To this end, we develop a structure-enhanced Deep Q-Network (SE-DQN) algorithm that leverages the monotone structure of the optimal policy to improve offline training performance.

\subsection{Deep Reinforcement Learning}
Solving the MDP formulated in Section \ref{Sec_MDP} requires the statistical knowledge of channels and harvested energy. However, this information may be unknown a priori in practice (e.g., edge devices are deployed in unknown or time-varying environments).
To address this challenge, the DQN algorithm has been widely adopted \cite{mnih2015human}.
DQN is an extension of the Q-learning algorithm, which uses a state-action value function, known as the Q-function, to estimate the quality of state-action pairs.
In the context of the MDP defined in Problem \ref{Problem1}, the update rule for the Q-function in Q-learning is given by
\begin{equation}\label{Q_equ}
	\begin{aligned}
		&Q(\v{h}_t,\v{b}_t,\v{p}_t)=c(\v{h}_t,\v{p}_t)+\min_{\v{p}_{t+ 1}\in \ca{P}_{t+1}}   Q(\v{h}_{t+1},\v{b}_{t+1},\v{p}_{t+1}).\\
	\end{aligned}    
\end{equation}
To effectively handle large state and action spaces, DQN approximates the Q-function using a neural network with parameters $\v{\theta}$, denoted as $Q(\v{h}_t,\v{b}_t,\v{p}_t;\v{\theta})$.
To train the DQN, a dataset of experiences is first collected by interacting with the environment, with each experience containing the current state, the action taken, the resulting cost, and the next state.
Then, the parameters $\v{\theta}$ of DQN are optimized by minimizing the temporal difference (TD) error, which is defined as
\begin{equation}\label{TD_error}
	\begin{aligned}
		\textup{TD} \triangleq y_t-Q(\v{h}_{t},\v{b}_{t},\v{p}_t;\v{\theta}),
	\end{aligned}    
\end{equation}
where $y_t$ is the target value defined as
\begin{equation}\label{y}
	\begin{aligned}
		y_t = c(\v{h}_t,\v{p}_t) + \min_{\v{p}_{t+1}\in \ca{P}_{t+1}} Q(\v{h}_{t+1},\v{b}_{t+1},\v{p}_{t+1};\v{\theta}^{\prime}),
	\end{aligned}    
\end{equation}
with $\v{\theta}^{\prime}$ being the parameters of a target network that stabilizes the training process. By minimizing the TD error, the DQN aligns its estimations with both immediate and future costs, guiding it towards the optimal policy.

\subsection{Algorithm Development}

The standard DQN algorithm dose not utilize the structural properties of the policy. In this subsection, we develop an SE-DQN algorithm that leverages the monotone structure obtained in Theorems \ref{p_increasing_b} and \ref{p_decreasing_h} to improve offline training performance. Specifically, we propose 1) an SE action selection method based on the monotone structure to select reasonable actions; and 2) an SE loss function to penalize actions that do not follow the monotone structure.
We present the SE action selection method and the SE loss function in the following.

\subsubsection{SE Action Selection}

The monotone structure implies that the optimal actions of a state and its adjacent states are related. Thus, we can infer the optimal action for a given state according to the actions taken by adjacent states.
At each time slot $t$, given the observed state $(\v{h}_t,\v{b}_t)$, we first compute the actions of its adjacent states as follows:
\begin{equation}\label{neighbor_actions}
	\begin{aligned}
		{\v{p}}^+_t = \argmin_{\v{p}_t\in \ca{P}_{t}}Q(\v{h}^-_t,\v{b}_t^+,\v{p}_t;\v{\theta}), \\
		{\v{p}}^-_t = \argmin_{\v{p}_t\in \ca{P}_{t}}Q(\v{h}^+_t,\v{b}_t^-,\v{p}_t;\v{\theta}),
	\end{aligned}
\end{equation}
where $\v{h}^-_t = \v{h}_t - \v{1}$, $\v{h}^+_t = \v{h}_t + \v{1}$, $\v{b}^-_t = \v{b}_t - \v{1}$, $\v{b}^+_t = \v{b}_t + \v{1}$, and $\v{1}$ denotes the all-ones vector. Then, we check whether there exist any actions $\hat{\v{p}}_t \in \ca{P}_t$ satisfying
\begin{equation}\label{exist_in_neighbor_actions}
	\begin{aligned}
		{\v{p}}^-_t \leq \hat{\v{p}}_t \leq {\v{p}}^+_t.
	\end{aligned}
\end{equation}
If such actions $\hat{\v{p}}_t$, referred to as the SE actions, exist, one of them is randomly selected for execution.
Otherwise, if no SE action exists, we execute the greedy action $\tilde{\v{p}}_t$, which is generated according to the standard DQN algorithm as follows:
\begin{equation}\label{greedy_action}
	\begin{aligned}
		\tilde{\v{p}}_t = \argmin_{\v{p}_t\in \ca{P}_{t}}Q(\v{h}_t,\v{b}_t,\v{p}_t;\v{\theta}).
	\end{aligned}
\end{equation}

\begin{algorithm}[tbp]
	\caption{Structure-Enhanced DQN Algorithm} \label{algorithm2}
	\textbf{Initialize:} policy network with random weights $\v{\theta}$ and target network with weights $\v{\theta}^{\prime} = \v{\theta}$\;
	\For{$\textup{epoch} = 1, 2, \ldots, E$}
	{
		Randomly initialize the state $(\v{h}_{1},\v{b}_{1})$\;
		\For{$t = 1, 2, \ldots, T$}
		{
			Observe the current state $(\v{h}_{t},\v{b}_{t})$\;
			Compute the actions of the adjacent states according to \eqref{neighbor_actions}\;
			\eIf{\textup{the SE actions $\hat{\v{p}}_t$ satisfying \eqref{exist_in_neighbor_actions} exist}}{
				Execute a randomly selected SE action\;
			}
			{
				Execute the greedy action $\tilde{\v{p}}_t$ in \eqref{greedy_action}\;
			}
			
			Observe the cost $c(\v{h}_{t},\v{p}_{t})$ and the next state $(\v{h}_{t+1},\v{b}_{t+1})$\;
			Calculate the TD error and AD error\;
			Update $\v{\theta}$ by performing SGD on the SE loss function $L(\v{\theta})$\;
			Update the target network $\v{\theta}^{\prime} = \v{\theta}$\;
		}
	}
\end{algorithm}

\subsubsection{SE Loss Function}
Different from the standard DQN, we develop an action difference (AD) error, which measures the difference between the Q-values of the SE action $\hat{\v{p}}_t$ and the greedy action $\tilde{\v{p}}_t$ as follows:
\begin{equation}\label{AD_error}
	\begin{aligned}
		\textup{AD} \triangleq Q(\v{h}_{t},\v{b}_{t},\hat{\v{p}}_t;\v{\theta})-Q(\v{h}_{t},\v{b}_{t},\tilde{\v{p}}_t;\v{\theta}).
	\end{aligned}    
\end{equation}
Since the optimal policy has a monotone structure, the SE action $\hat{\v{p}}_t$ should be close to the greedy action $\tilde{\v{p}}_t$.
Thus, DQN should minimize the AD error. By combining the TD error and AD error, we define the SE loss function as
\begin{equation}\label{SE_loss}
	\begin{aligned}
		L(\v{\theta}) = \alpha\textup{TD}^2 + (1-\alpha)\textup{AD}^2,
	\end{aligned}    
\end{equation}
where $\alpha\in (0,1)$ is a hyperparameter to balance the importance of the TD error and AD error. To optimize $\v{\theta}$, we minimize the SE loss function $L(\v{\theta})$ by adopting stochastic gradient descent (SGD). The proposed SE-DQN algorithm is summarized in Algorithm \ref{algorithm2}. In Section \ref{Sim_SE_DQN}, the performance of the proposed SE-DQN algorithm is evaluated and compared with the standard DQN algorithm.

\section{Simulation Results}\label{sim_result}

In this section, we validate the theoretical results and evaluate the performance of the proposed algorithms through numerical experiments on real datasets.

\subsection{Simulation Setups}

\subsubsection{Learning Model Setting}
The proposed algorithms are evaluated by implementing the image classification tasks on the MNIST and CIFAR-10 datasets. For the MNIST dataset, we employ the regularized multinomial logistic regression model, with the loss function defined as follows:
\begin{equation}\label{logistic_loss}
	\begin{aligned}
		&f(\v{w};\v{x}_{n,i},y_{n,i})\\
		=&\!-\!\sum_{c=1}^C \mathbbm{1}_{\!\left\lbrace y_{n,i}=c\right\rbrace }\! \log\! \left(\!\frac{\exp\! \left(\boldsymbol{w}_c^{\s{T}} \boldsymbol{x}_{n,i}\right)}{\sum_{j=1}^C \exp \!\left(\boldsymbol{w}_j^{\s{T}} \boldsymbol{x}_{n,i}\right)}\!\right) \!+\! \frac{\alpha_r}{2}\sum_{c=1}^C\|\boldsymbol{w}_c \|^2,
	\end{aligned}    
\end{equation}
where $C$ represents the total number of label categories in the dataset, $\boldsymbol{w}_c$ denotes the model parameter vector associated with label category $c\in [C]$, and $\alpha_r=0.01$. For the CIFAR-10 dataset, we train a convolutional neural network (CNN) with the following structure: four $3 \times 3$ convolution layers each with 64 channels and followed by a $2 \times 2$ max pooling layer, a fully connected layer with 120 units, and finally a 10-unit softmax output layer.
Each convolution or fully connected layer is activated by a ReLU function.
Both the MNIST and CIFAR-10 datasets are divided into $N$ subsets, with each subset assigned to a distinct device as its local training dataset.
We investigate both i.i.d. and non-i.i.d. data distributions. For i.i.d. scenarios, we implement uniform random data distribution across devices. To simulate realistic heterogeneity in non-i.i.d. settings, we employ a Dirichlet distribution Dir(0.8) for both the FMNIST and CIFAR-10 datasets.
We set the learning rate $\eta$ as 0.01 for the MNIST dataset and 0.05 for the CIFAR-10 dataset.

\subsubsection{Communication Model Setting}
The wireless channels from edge devices to the server follow Rayleigh fading and can be modeled as finite-state Markov channels \cite{tan2000first}.
It is assumed that the current channel state either remains constant or transitions to adjacent states in the next time slot, which is suitable for slow fading channels.
The channel transition probability from state $H_i$ to state $H_{j}$ can be approximated as follows \cite{sadeghi2008finite}:
\begin{equation}\label{sim_channel_trans}
	\begin{aligned}
		&\b{P}(h_{n,t+1}=H_{j}|h_{n,t}=H_{i})\\
		&=\begin{cases}\frac{Z(H_{i+1})\tau}{\b{P}({H_{i}})} , & \text { if } j = i + 1, \\   \frac{Z(H_{i})\tau}{\b{P}({H_{i}})} , & \text { if } j = i - 1,\\ 1 -  \frac{Z(H_{i+1})\tau}{\b{P}({H_{i}})} - \frac{Z(H_{i})\tau}{\b{P}({H_{i}})} , & \text { if } j = i,
		\end{cases}
	\end{aligned}    
\end{equation}
where $Z(H_{i})$ is the level crossing rate of channel state $H_i$, $\b{P}({H_{i}})$ is the steady probability of channel state $H_i$, and $\tau$ is the duration of one time slot.

For the energy harvesting model, we employ a realistic data record of the irradiance (i.e., the intensity of the solar radiation in units $\mu$W/$\text{cm}^2$) for the month of June from 2008 to 2010, measured at a solar site in Elizabeth City State University \cite{Andreas1981}.
Since edge devices are small, the solar panel area is set as 1 $\text{cm}^2$.
The energy conversion efficiency is assumed to be $20\%$ and we assume that the battery state is randomly initialized.

\begin{figure}[tbp]
	\vspace{-5pt}
	\renewcommand\figurename{\small Fig.}
	\centering \vspace*{6pt} \setlength{\baselineskip}{10pt}
	\hspace{-14.5pt}
	\includegraphics[width = 0.464\textwidth,center,trim=0 10 0 0,clip]{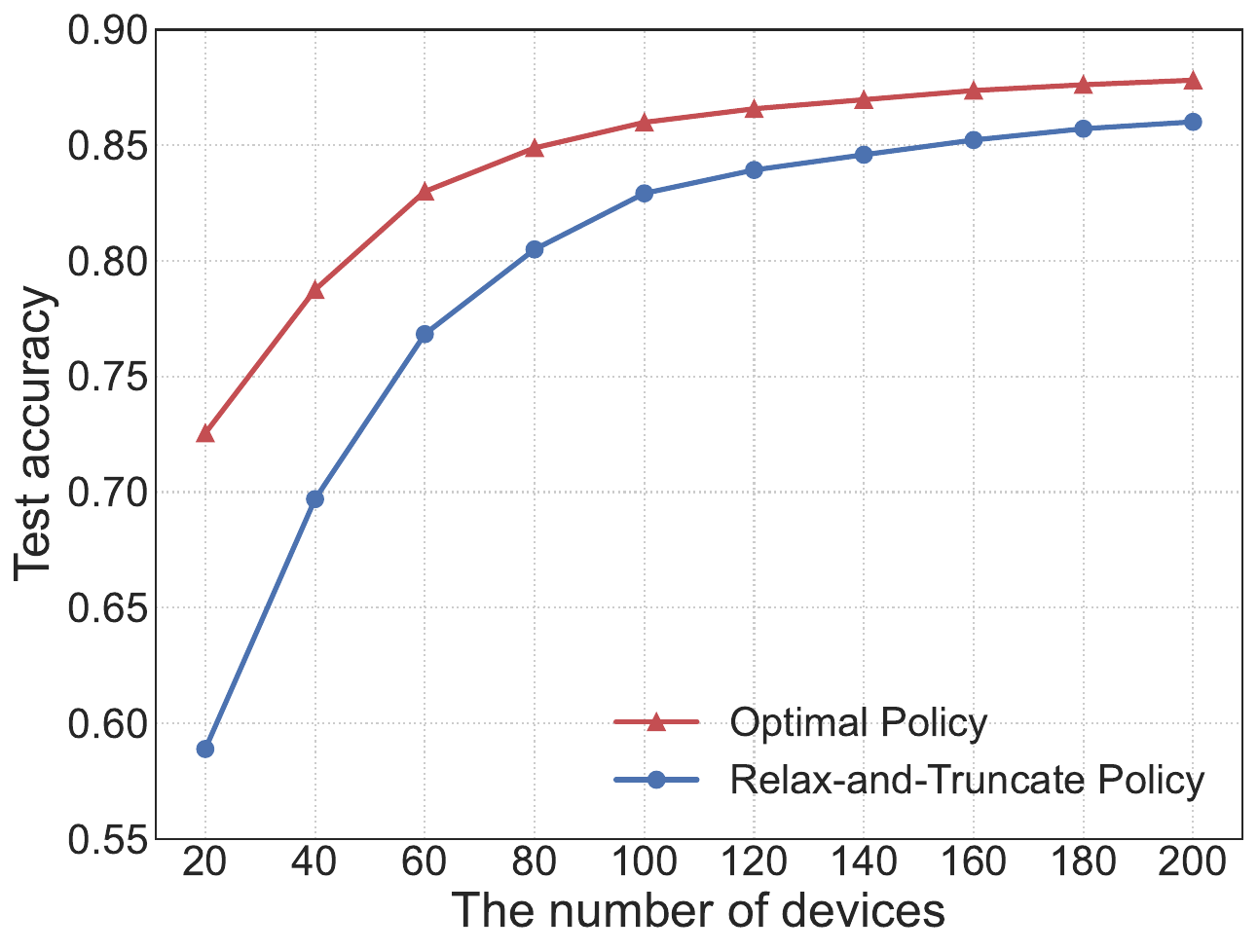}
	\caption{Performance comparison of the relax-and-truncate policy and the optimal policy with $R = 0.4N$.}\label{fig_jianjin}
	\vspace{-8pt}
\end{figure}

\begin{figure}[!t]
	\renewcommand\figurename{\small Fig.}
	\centering \vspace*{4pt} \setlength{\baselineskip}{10pt}
	\hspace{-4.5pt}
	\includegraphics[width = 0.467\textwidth,center,trim=0 7 0 0,clip]{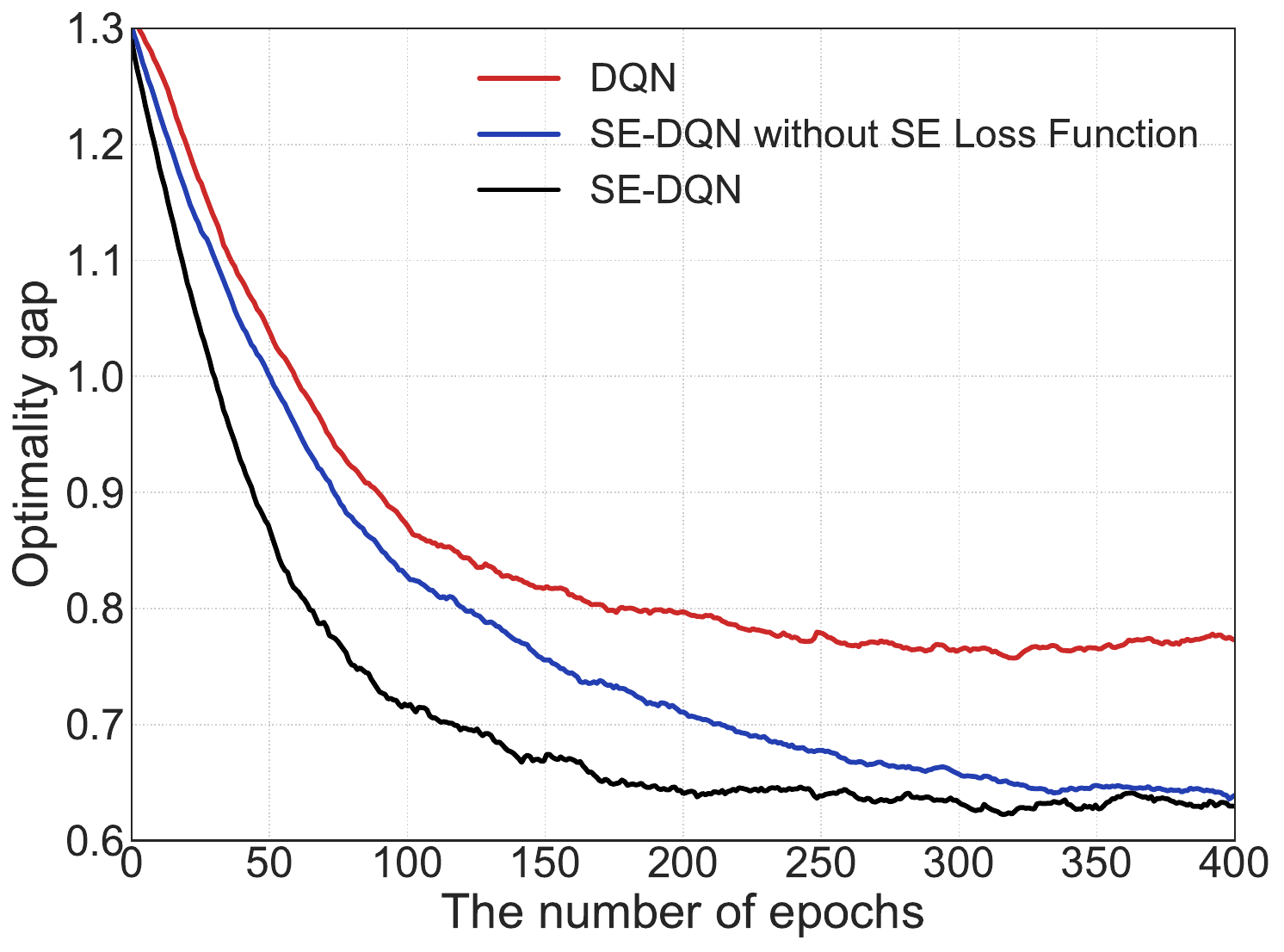}
	\vspace{-8pt}
	\caption{Offline training performance comparison of DQN, SE-DQN, and SE-DQN without SE loss function.}\label{fig_dqn}
	\vspace{-8pt}
\end{figure}

\begin{figure*}[!tbp]                
	\renewcommand\figurename{\small Fig.}  
	\centering
	\vspace{-10pt}
	\setlength{\abovecaptionskip}{8pt}
	\setlength{\belowcaptionskip}{5pt}
	
	\subfigure[%
	{\normalsize i.i.d.}%
	\label{fig:mnist}]{%
		\includegraphics[width=0.48\textwidth,trim=0 11 0 0,clip]{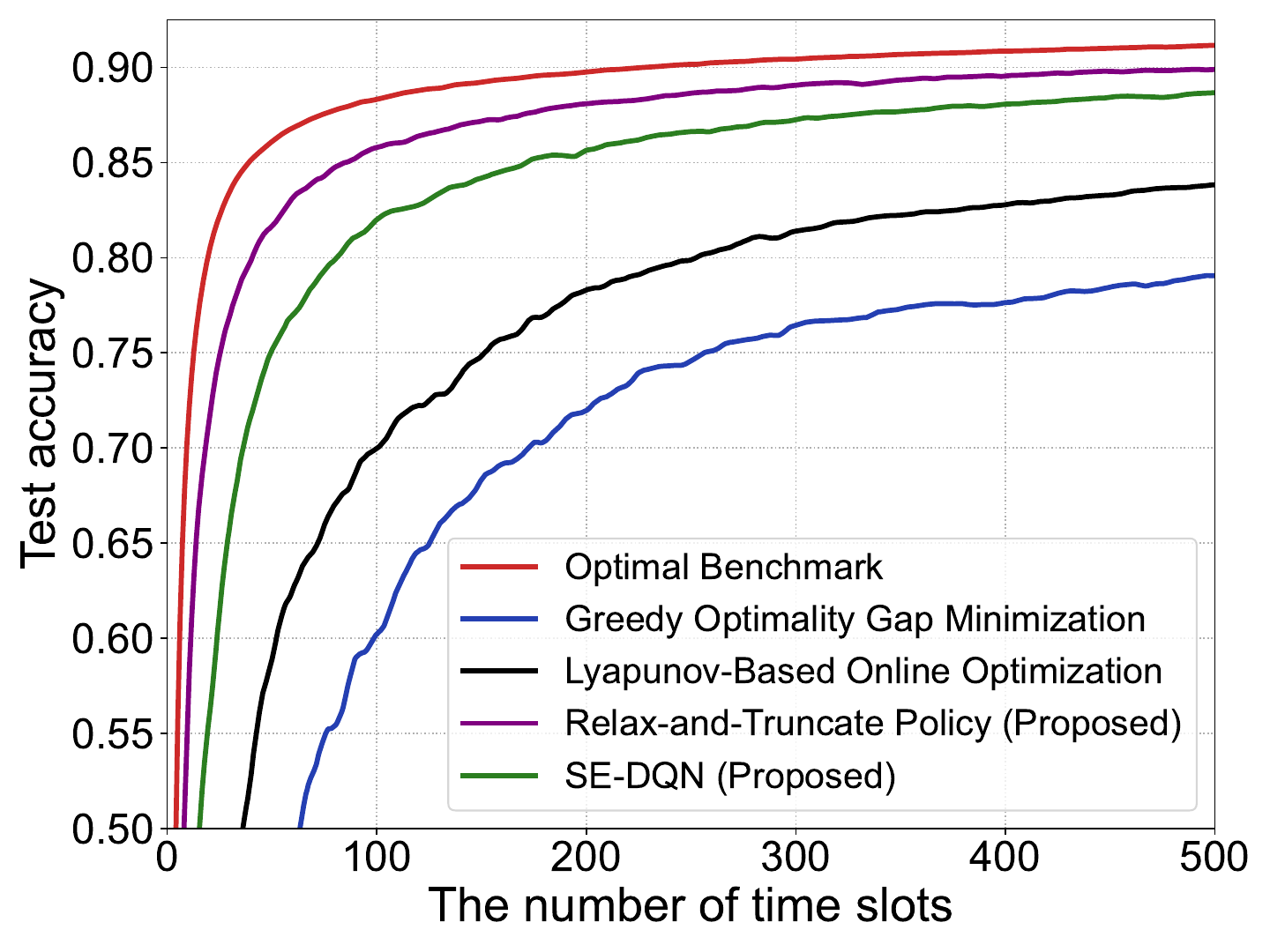}}
	\hfill
	\subfigure[%
	{\normalsize non-i.i.d.}%
	\label{fig:cifar}]{%
		\includegraphics[width=0.48\textwidth,trim=0 11 0 0,clip]{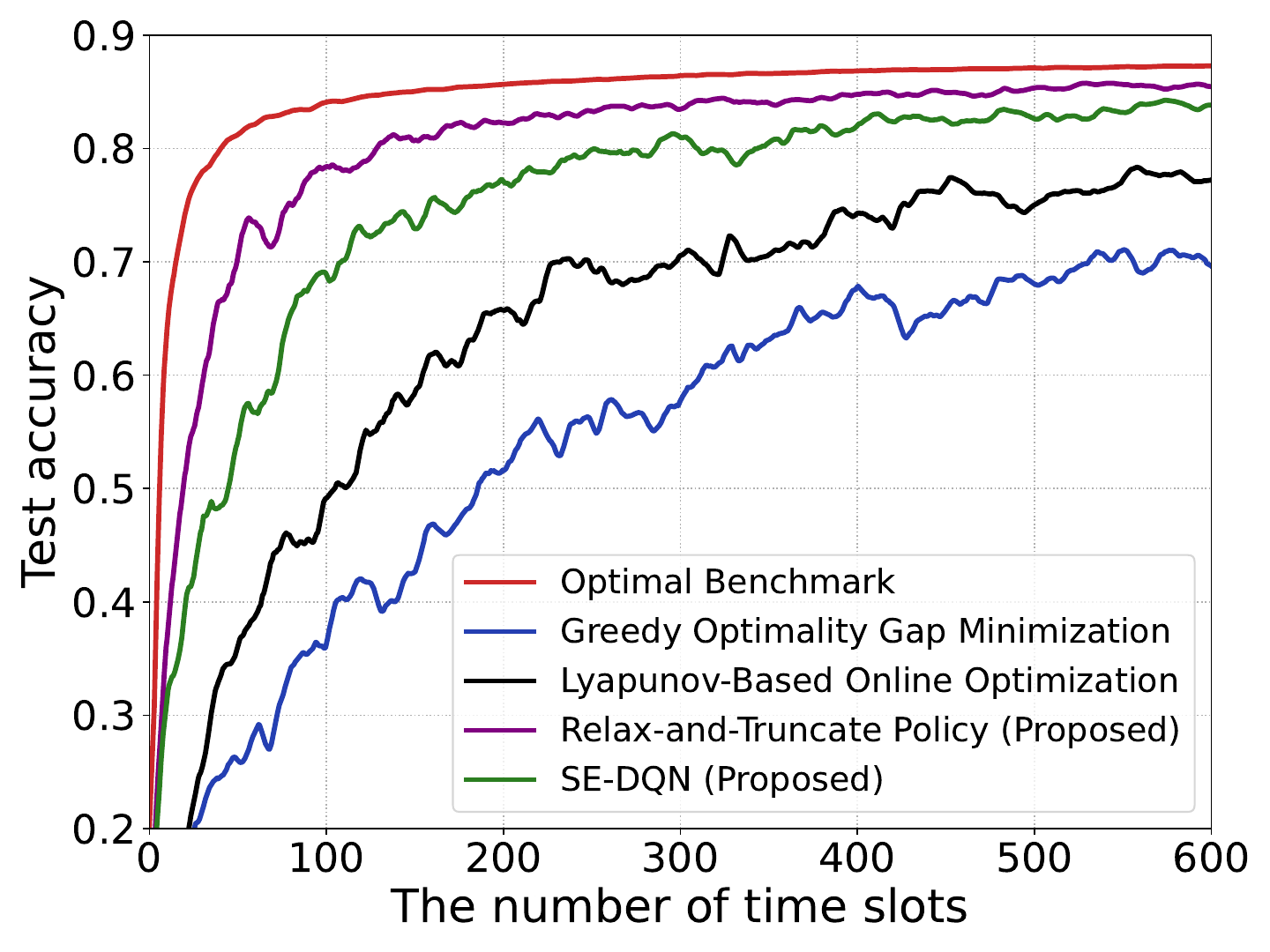}}%
	\vspace{-8pt}
	\caption{Test accuracy versus the number of time slots under different transmission strategies on the MNIST dataset: (a) i.i.d. data distribution and (b) non-i.i.d. data distribution.}
	\label{MNIST}
	\vspace{-10pt}
\end{figure*}

\subsection{Asymptotically Optimal Performance of the Relax-and-Truncate Policy}

In this subsection, we empirically validate the asymptotic optimality of the relax-and-truncate policy, as established in Theorem \ref{asymptotically_optimal}.
We conduct a classification task on the MNIST dataset to compare the learning performance of the relax-and-truncate policy with the optimal policy.
Fig. \ref{fig_jianjin} demonstrates the test accuracy for the number of devices $N$ ranging from 20 to 200 with a bandwidth of $R = 0.4N$. As observed, the performance gap between these two policies diminishes as $N$ increases, which corroborates our theoretical result.

\subsection{Offline Training Performance of the SE-DQN Algorithm}\label{Sim_SE_DQN}

In this subsection, we compare the offline training performance of the proposed SE-DQN algorithm with the standard DQN algorithm.
Moreover, to demonstrate the necessity of both SE action selection and SE loss function in the SE-DQN algorithm, we also compare the complete SE-DQN with a variant that disables the SE loss function (i.e., setting $\alpha=1$ in \eqref{SE_loss}).
The neural networks of DQN and SE-DQN have identical hyperparameters, each consisting of two fully connected layers with 256 units.
Fig. \ref{fig_dqn} illustrates the optimality gap (i.e., the objective value of Problem \ref{Problem1}) achieved by DQN, SE-DQN, and SE-DQN without the SE loss function. As observed, SE-DQN reduces the optimality gap by $25\%$ compared to the DQN algorithm. Furthermore, disabling the SE loss function doubles the training convergence time of the SE-DQN algorithm. Thus, both SE action selection and SE loss function in the SE-DQN algorithm are essential for guaranteeing its offline training performance.

\subsection{Comparison With Benchmark Schemes}

In this subsection, we demonstrate the superiority of the proposed algorithms to the following three benchmark schemes.

\begin{itemize}
	\item \textbf{Optimal Benchmark}: All devices have sufficient energy and bandwidth to upload local gradients without packet drops. This benchmark achieves the ideal performance by completely ignoring the negative effects of wireless channels. However, it is infeasible in practical wireless systems.
	\item \textbf{Greedy Optimality Gap Minimization}: At each time slot $t$, the server schedules all devices that have remaining battery energy.
	If the number of eligible devices exceeds $R$, the server selects the $R$ devices with best channel conditions. Each scheduled device depletes its remaining energy to upload the local gradient.

	\item \textbf{Lyapunov-Based Online Optimization}: This benchmark employs Lyapunov optimization framework for online optimization \cite{an_online_2023}. 
	At each time slot $t$, it adopts Gibbs sampling for device scheduling. Then, each scheduled device determines its transmission power by solving the per-time-slot subproblem as follows:
	\begin{equation}\label{online_opt}
		\begin{aligned}
			\underset{p_{n,t}}{\text{min}} ~~&\frac{ 2 V L G^2 q_{n,t}(h_{n,t},p_{n,t}) }{N \sqrt{KT}  }+ Q_{n,t}e_{n,t}\\
			\text{s.t.} \;\;\; & \;0\leq e_{n,t}\leq b_{n,t},
		\end{aligned}
	\end{equation}
	where $V>0$ is a tradeoff hyperparameter, and $Q_{n,t}$ is the virtual queue of device $n$. The queuing dynamics of $Q_{n,t}$ is given by $Q_{n,t+1}= Q_{n,t}+e_{n,t}-u_{n,t}.$
	
\end{itemize}


%

In Fig. \ref{MNIST}, we compare the learning performance of the proposed algorithms with benchmark schemes on the MNIST dataset under both i.i.d. and non-i.i.d. data distributions.
As observed, the proposed algorithms achieve near-optimal learning performance and outperform other benchmarks.
The greedy optimality gap minimization scheme myopically uses up all remaining energy at each time slot, leading to poor performance in subsequent time slots.
The Lyapunov-based online optimization scheme does not utilize the statistical knowledge of system dynamics, e.g., the distributions of the channel gains and the amount of harvested energy, resulting in poor long-term performance.
In contrast, the proposed algorithms leverage the previously known or learned knowledge of channel and energy dynamics to design the transmission policy from a long-term perspective, thereby improving the learning performance.

Fig. \ref{cifar} compares the proposed algorithms with benchmark methods on the CIFAR-10 dataset.
To illustrate the broad applicability of the proposed methods, we implement CNNs as the classifier models, which do not satisfy Assumption \ref{smoothness_assump}. Despite this relaxation, the empirical results exhibit consistent performance trends, with the proposed algorithms significantly outperforming other benchmarks and closely approaching the optimal benchmark. This highlights the robustness and effectiveness of the proposed approaches under various FL scenarios.

\begin{figure}[tbp]
	\hspace{0pt}
	\renewcommand\figurename{\small Fig.}
	\centering \vspace*{8pt} \setlength{\baselineskip}{10pt}
	\hspace{-5pt}
	\includegraphics[width = 0.485\textwidth,center,trim=0 0 0 0,clip]{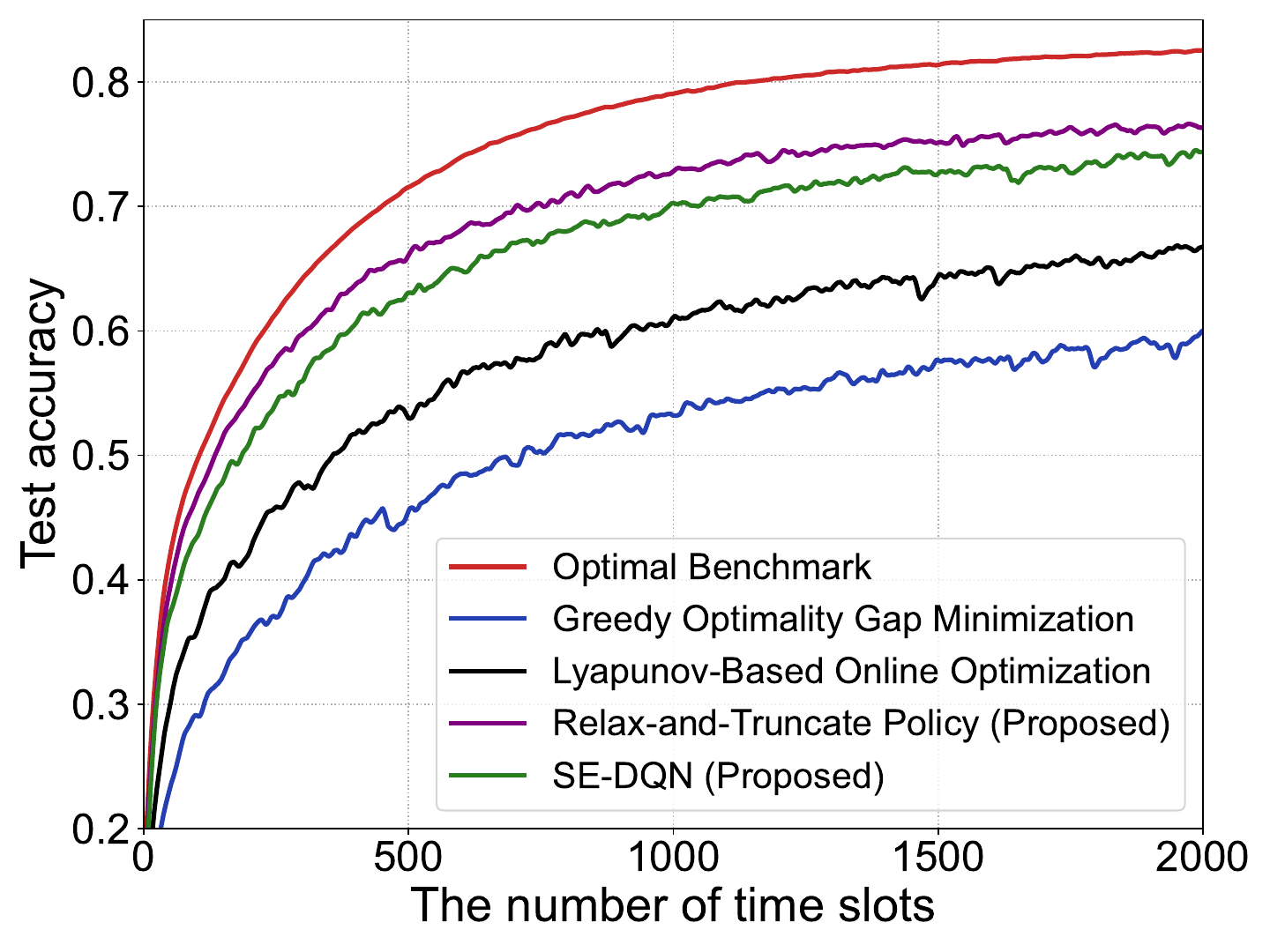}
	\vspace{-25pt}
	\caption{Test accuracy versus the number of time slots under different transmission strategies on the CIFAR-10 dataset.}\label{cifar}
	\vspace{-5pt}
\end{figure}

\begin{figure}[tbp]
	\renewcommand\figurename{\small Fig.}
	\centering \vspace*{0pt} \setlength{\baselineskip}{10pt}
	\hspace{-13pt}
	\includegraphics[width = 0.47\textwidth,center,trim=0 0 0 0,clip]{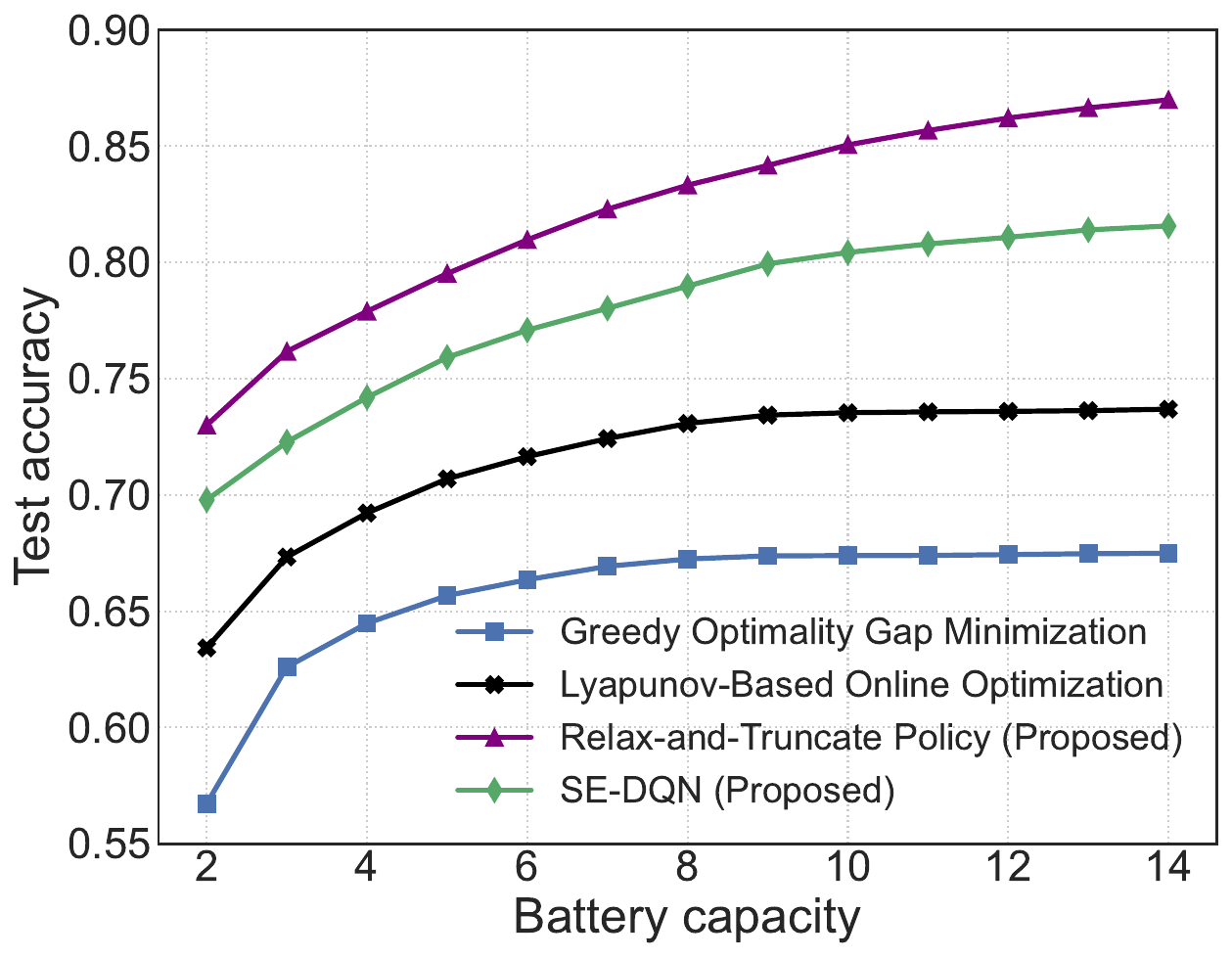}
	\vspace{-7pt}
	\caption{Test accuracy versus battery capacity under different transmission strategies on the MNIST dataset.}\label{fig_battery}
	\vspace{-8pt}
\end{figure}

In Fig. \ref{fig_battery}, we compare the proposed algorithms with other benchmarks under different battery capacities on the MNIST dataset.
As can be observed, our proposed algorithms exhibit considerable performance improvements under different battery capacities.
For all schemes, the test accuracy increases with battery capacity, since smaller capacities lead to quicker saturation and potential waste of harvested energy.
However, once the battery capacity exceeds 9, the accuracy of other benchmarks plateaus, whereas the proposed algorithms continue to exhibit improvement. This divergence can be attributed to the benchmarks' neglect of battery capacity limitations in their algorithm design.
In contrast, the proposed algorithms leverage increased battery capacity to improve state representation and decision-making flexibility, thereby enhancing the learning performance.

\section{Conclusion}

In this paper, we have studied FL with EH devices.
We conducted a convergence analysis to demonstrate the influence of transmission policy on the learning performance.
Based on the convergence bound, we have formulated a joint device scheduling and power control problem to minimize the optimality gap under energy and bandwidth constraints.
We have further modeled this optimization problem as an MDP and demonstrated that its optimal policy possesses a monotone structure.
We have also proposed a low-complexity relax-and-truncate algorithm that is asymptotically optimal as the number of devices increases.
Moreover, for unknown channels and harvested energy statistics, we have developed an SE-DQN algorithm that leverages the monotone structure of the optimal policy to improve the training performance. Finally, extensive numerical experiments have been presented to validate the theoretical results and the effectiveness of the proposed algorithms.

\section*{Appendix}

\subsection{Proof of Theorem \ref{thm_convergence_bound}}\label{proof_convergence}


According to Assumption \ref{smoothness_assump}, we have
\begin{equation}\label{new_1}
	\begin{aligned}
		&F_n(\bar{\v{w}}_t^{(k+1)}) \\
		\aleq & F_n({\v{w}}_{n,t}^{(k)})+
		\bigl\langle 
		\nabla F_n({\v{w}}_{n,t}^{(k)}),\bar{\v{w}}_t^{(k+1)} - {\v{w}}_{n,t}^{(k)}
		\bigr\rangle\\
		&+\frac{L}{2}\bigl\|\bar{\v{w}}_t^{(k+1)} - {\v{w}}_{n,t}^{(k)}\bigr\|^2 \\
		\bleq & F_n({\v{w}}^{*})+
		\bigl\langle 
		\nabla F_n({\v{w}}_{n,t}^{(k)}),\bar{\v{w}}_t^{(k+1)} - {\v{w}}^{*}
		\bigr\rangle\\
		&+\frac{L}{2}\bigl\|\bar{\v{w}}_t^{(k+1)} - {\v{w}}_{n,t}^{(k)}\bigr\|^2 \\
		\leq & F_n({\v{w}}^{*})+
		\bigl\langle 
		\nabla F_n({\v{w}}_{n,t}^{(k)}),\bar{\v{w}}_t^{(k+1)} - {\v{w}}^{*}
		\bigr\rangle\\
		&+L \bigl\|\bar{\v{w}}_t^{(k+1)} - \bar{\v{w}}_t^{(k)}\bigr\|^2 + L \bigl\|\bar{\v{w}}_t^{(k)} - {\v{w}}_{n,t}^{(k)}\bigr\|^2,
	\end{aligned}    
\end{equation}
where (a) and (b) follow from the $L$-smoothness and convexity of the local loss functions, respectively.
Then, according to the local model update rule \eqref{eq:local_update}, we have
\begin{equation}\label{new_2}
	\begin{aligned}
		&\frac{1}{N} \sum_{n=1}^N \left\langle \v{g}_{n,t}^{(k)}, \bar{\v{w}}_t^{(k+1)} - {\v{w}}^{*} \right\rangle\\
		= & - \frac{1}{\eta} \left\langle \bar{\v{w}}_t^{(k+1)} - \bar{\v{w}}_t^{(k)}, \bar{\v{w}}_t^{(k+1)} - {\v{w}}^{*} \right\rangle \\
		 = &\frac{1}{2\eta}
		\bigl\| \bar{\v{w}}_t^{(k)} - {\v{w}}^{*} \bigr\|^2
		-\frac{1}{2\eta}
		\bigl\| \bar{\v{w}}_t^{(k+1)} - \bar{\v{w}}_t^{(k)} \bigr\|^2
		\\
		&-\frac{1}{2\eta}
		\bigl\| \bar{\v{w}}_t^{(k+1)} - {\v{w}}^{*} \bigr\|^2.
	\end{aligned}    
\end{equation}
Combining \eqref{new_2} into \eqref{new_1}, it follows that
\begin{equation}\label{new_3}
	\begin{aligned}
		&F(\bar{\v{w}}_t^{(k+1)}) -  F({\v{w}}^{*})\\
		= & \frac{1}{N} \sum_{n=1}^{N} \left( F_n(\bar{\v{w}}_t^{(k+1)}) -  F_n({\v{w}}^{*}) \right) \\
		\leq & \frac{1}{N} \sum_{n=1}^{N} 
		\bigl\langle 
		\nabla F_n({\v{w}}_{n,t}^{(k)}),\bar{\v{w}}_t^{(k+1)} - {\v{w}}^{*}
		\bigr\rangle\\
		&+L \bigl\|\bar{\v{w}}_t^{(k+1)} - \bar{\v{w}}_t^{(k)}\bigr\|^2 + \frac{L}{N} \sum_{n=1}^{N} \bigl\|\bar{\v{w}}_t^{(k)} - {\v{w}}_{n,t}^{(k)}\bigr\|^2\\
		\leq &  \frac{1}{N} \sum_{n=1}^{N}\bigl\langle 
		\nabla F_n({\v{w}}_{n,t}^{(k)}) - \v{g}_{n,t}^{(k)} ,\bar{\v{w}}_t^{(k+1)} - {\v{w}}^{*}\bigr\rangle \\
		&+ L \bigl\|\bar{\v{w}}_t^{(k+1)} - \bar{\v{w}}_t^{(k)}\bigr\|^2  + \frac{L}{N} \sum_{n=1}^{N}\bigl\|\bar{\v{w}}_t^{(k)} - {\v{w}}_{n,t}^{(k)}\bigr\|^2\\
		&+\frac{1}{2\eta}
		\bigl\| \bar{\v{w}}_t^{(k)} - {\v{w}}^{*} \bigr\|^2
		-\frac{1}{2\eta}
		\bigl\| \bar{\v{w}}_t^{(k+1)} - \bar{\v{w}}_t^{(k)} \bigr\|^2
		\\
		&-\frac{1}{2\eta}
		\bigl\| \bar{\v{w}}_t^{(k+1)} - {\v{w}}^{*} \bigr\|^2.
	\end{aligned}    
\end{equation}
Moreover, given the unbiased stochastic gradient assumption in Assumption \ref{ass_bounded_var_assump}, i.e., $\mathbb{E}\left[ \nabla F_n({\v{w}}_{n,t}^{(k)}) - \v{g}_{n,t}^{(k)}  \right] = 0$, we have
	\begin{align}\label{new_4}
		&\frac{1}{N} \sum_{n=1}^{N}\mathbb{E}\left[ \bigl\langle 
		\nabla F_n({\v{w}}_{n,t}^{(k)}) - \v{g}_{n,t}^{(k)} ,\bar{\v{w}}_t^{(k+1)} - {\v{w}}^{*}\bigr\rangle\right]  \nonumber \\
		= & \frac{1}{N} \sum_{n=1}^{N}\mathbb{E}\left[ \bigl\langle 
		\nabla F_n({\v{w}}_{n,t}^{(k)}) - \v{g}_{n,t}^{(k)} ,\bar{\v{w}}_t^{(k+1)} - \bar{\v{w}}_t^{(k)}\bigr\rangle\right]  \nonumber\\
		\leq & \eta  \mathbb{E}\left[  \left\|  \frac{1}{N} \sum_{n=1}^{N}  \left( \nabla F_n({\v{w}}_{n,t}^{(k)}) - \v{g}_{n,t}^{(k)} \right)  \right\|^2  \right]  \nonumber \\
		&+ \frac{1}{4\eta}  \mathbb{E}\left[  \left\| \bar{\v{w}}_t^{(k+1)} - \bar{\v{w}}_t^{(k)}   \right\|^2  \right]  \nonumber \\
		\leq & \frac{\eta \sigma_l^2}{N} + \frac{1}{4\eta}  \mathbb{E}\left[  \left\| \bar{\v{w}}_t^{(k+1)} - \bar{\v{w}}_t^{(k)}   \right\|^2  \right] .
	\end{align} 
Plugging \eqref{new_4} back to \eqref{new_3} yields
	\begin{align}\label{new_5}
		&\mathbb{E}\left[ F(\bar{\v{w}}_t^{(k+1)}) -  F({\v{w}}^{*}) \right]  \nonumber\\
		\leq &  \frac{\eta \sigma_l^2}{N} + \frac{1}{4\eta}  \mathbb{E}\left[  \left\| \bar{\v{w}}_t^{(k+1)} - \bar{\v{w}}_t^{(k)}   \right\|^2  \right] \nonumber \\
		&+ L \mathbb{E}\left[ \bigl\|\bar{\v{w}}_t^{(k+1)} - \bar{\v{w}}_t^{(k)}\bigr\|^2  \right] + \frac{L}{N} \sum_{n=1}^{N} \mathbb{E}\left[ \bigl\|\bar{\v{w}}_t^{(k)} - {\v{w}}_{n,t}^{(k)}\bigr\|^2 \right]  \nonumber\\
		&+\frac{1}{2\eta} \mathbb{E}\left[ 
		\bigl\| \bar{\v{w}}_t^{(k)} - {\v{w}}^{*} \bigr\|^2\right] 
		-\frac{1}{2\eta} \mathbb{E}\left[ 
		\bigl\| \bar{\v{w}}_t^{(k+1)} - \bar{\v{w}}_t^{(k)} \bigr\|^2 \right] 
		 \nonumber\\
		&-\frac{1}{2\eta} \mathbb{E}\left[ 
		\bigl\| \bar{\v{w}}_t^{(k+1)} - {\v{w}}^{*} \bigr\|^2  \right]  \nonumber\\
		\leq &  \frac{\eta \sigma_l^2}{N} + \left(L - \frac{1}{4\eta} \right)  \mathbb{E}\left[  \left\| \bar{\v{w}}_t^{(k+1)} - \bar{\v{w}}_t^{(k)}   \right\|^2  \right]  \nonumber\\
		&+ \frac{L}{N} \sum_{n=1}^{N} \mathbb{E}\left[ \bigl\|\bar{\v{w}}_t^{(k)} - {\v{w}}_{n,t}^{(k)}\bigr\|^2 \right]  \nonumber\\
		&+\frac{1}{2\eta} \mathbb{E}\left[ 
		\bigl\| \bar{\v{w}}_t^{(k)} - {\v{w}}^{*} \bigr\|^2\right] 
		-\frac{1}{2\eta} \mathbb{E}\left[ 
		\bigl\| \bar{\v{w}}_t^{(k+1)} - {\v{w}}^{*} \bigr\|^2  \right]  \nonumber\\
		\aleq &  \frac{\eta \sigma_l^2}{N} + \frac{L}{N} \sum_{n=1}^{N} \mathbb{E}\left[ \bigl\|\bar{\v{w}}_t^{(k)} - {\v{w}}_{n,t}^{(k)}\bigr\|^2 \right]  \nonumber\\
		&+\!\frac{1}{2\eta} \mathbb{E}\left[ 
		\bigl\| \bar{\v{w}}_t^{(k)} - {\v{w}}^{*} \bigr\|^2\right] 
		\!-\!\frac{1}{2\eta} \mathbb{E}\left[ 
		\bigl\| \bar{\v{w}}_t^{(k+1)} - {\v{w}}^{*} \bigr\|^2  \right] ,
	\end{align}
where (a) holds due to $\eta \leq \frac{1}{4L}$.
Summing the inequality \eqref{new_5} from $k=1$ to $K$, it follows that
	\begin{align}\label{new_6}
		&\frac{1}{K} \sum_{k=1}^{K} \mathbb{E}\left[ F(\bar{\v{w}}_t^{(k+1)}) -  F({\v{w}}^{*}) \right] \nonumber\\
		\leq &  \frac{\eta \sigma_l^2}{N} + \frac{L}{NK} \sum_{n=1}^{N} \sum_{k=1}^{K} \mathbb{E}\left[ \bigl\|\bar{\v{w}}_t^{(k)} - {\v{w}}_{n,t}^{(k)}\bigr\|^2 \right] \nonumber\\
		&+\frac{1}{2\eta K} \sum_{k=1}^{K}\left(  \mathbb{E}\left[ 
		\bigl\| \bar{\v{w}}_t^{(k)} - {\v{w}}^{*} \bigr\|^2\right] 
		-\mathbb{E}\left[ 
		\bigl\| \bar{\v{w}}_t^{(k+1)} - {\v{w}}^{*} \bigr\|^2  \right] \right)\nonumber\\
		\leq &  \frac{\eta \sigma_l^2}{N} + \frac{L}{NK} \sum_{n=1}^{N} \sum_{k=1}^{K} \mathbb{E}\left[ \bigl\|\bar{\v{w}}_t^{(k)} - {\v{w}}_{n,t}^{(k)}\bigr\|^2 \right] \nonumber\\
		&+\frac{1}{2\eta K} \left(  \mathbb{E}\left[ 
		\bigl\| \bar{\v{w}}_t^{(1)} - {\v{w}}^{*} \bigr\|^2\right] 
		-\mathbb{E}\left[ 
		\bigl\| \bar{\v{w}}_t^{(K+1)} - {\v{w}}^{*} \bigr\|^2  \right] \right) .
	\end{align}
Then, according to Lemma 2 in \cite{wang2021field}, by taking $\eta \leq \frac{1}{4L}$, we have
\begin{equation}\label{new_7}
	\begin{aligned}
		&\mathbb{E}\left[ \bigl\|\bar{\v{w}}_t^{(k)} - {\v{w}}_{n,t}^{(k)}\bigr\|^2 \right] 
		\leq  18K^2\eta^2 \sigma_g^2 + 4 K \eta^2 \sigma_l^2.
	\end{aligned}    
\end{equation}
By combining \eqref{new_7} with \eqref{new_6} and telescoping over $t=1,\dots,T$, we have
	\begin{align}
		&\mathbb{E}\Biggl[\frac{1}{K T}\sum_{t=1}^{T}\sum_{k=1}^{K}
		\bigl(F(\bar{\v{w}}_t^{(k)}) - F(\v{w}^\star)\bigr)\Biggr] \nonumber\\
		\leq & \frac{\eta \sigma_l^2}{N}  + 4K\eta^2L\sigma_l^2 + 18K^2\eta^2L\sigma_g^2 \nonumber\\
		& +\frac{1}{2\eta K T} \sum_{t=1}^{T} \left(  \mathbb{E}\!\left[ 
		\bigl\| \bar{\v{w}}_t^{(1)} \!-\! {\v{w}}^{*} \bigr\|^2\right] 
		\!-\!\mathbb{E}\!\left[ 
		\bigl\| \bar{\v{w}}_t^{(K+1)} \!-\! {\v{w}}^{*} \bigr\|^2  \right] \right)   \nonumber\\
		=  & \frac{\eta \sigma_l^2}{N}  + 4K\eta^2L\sigma_l^2 + 18K^2\eta^2L\sigma_g^2 \nonumber\\
		& +\frac{1}{2\eta K T} \!\sum_{t=1}^{T}\! \left(  \mathbb{E}\!\left[ 
		\bigl\| \bar{\v{w}}_t^{(1)} \!-\! {\v{w}}^{*} \bigr\|^2\right] 
		\!-\!\mathbb{E}\!\left[ 
		\bigl\| {\v{w}}_{t+1} \!-\! {\v{w}}^{*} \bigr\|^2  \right] \right)   \nonumber\\
		& +\frac{1}{2\eta K T}\! \sum_{t=1}^{T} \!\left(  \mathbb{E}\!\left[ 
		\bigl\| {\v{w}}_{t+1} \!-\! {\v{w}}^{*} \bigr\|^2\right] 
		\!-\!\mathbb{E}\!\left[ 
		\bigl\| \bar{\v{w}}_t^{(K+1)} \!-\! {\v{w}}^{*} \bigr\|^2  \right] \right)   \nonumber\\
		\aleq &\frac{\eta\sigma_l^2}{N}
		+ 4K\eta^2L\sigma_l^2
		+ 18K^2\eta^2L\sigma_g^2  \nonumber\\
		&+ \frac{\|  \v{w}_1 - \v{w}^* \|^2}{2\eta K T} - \frac{1}{2\eta K T}\mathbb{E}\left[\|  \v{w}_{T+1} - \v{w}^* \|^2 \right]  \nonumber\\
		& +\frac{ G^2}{2\eta N K T}\sum_{t=1}^{T}\sum_{n=1}^{N}\left( 1 - \beta_{n,t} ( 1 -q_{n,t}(h_{n,t},p_{n,t})) \right)  \nonumber\\
		\leq & \frac{\|  \v{w}_1 - \v{w}^* \|^2}{2\eta K T}  + \frac{\eta\sigma_l^2}{N}
		+ 4K\eta^2L\sigma_l^2
		+ 18K^2\eta^2L\sigma_g^2  \nonumber\\
		& +\frac{ G^2}{2\eta N K T}\sum_{t=1}^{T}\sum_{n=1}^{N}\left( 1 - \beta_{n,t} ( 1 -q_{n,t}(h_{n,t},p_{n,t})) \right),\nonumber\\
	\end{align}
where (a) follows directly from the global model update rule defined in \eqref{update_global} and the corresponding broadcasting mechanism.
Then, by taking $\eta = \frac{1}{4 L \sqrt{K T}}$, we have
\begin{equation}
	\begin{aligned}
		&\mathbb{E}\Biggl[\frac{1}{K T}\sum_{t=1}^{T}\sum_{k=1}^{K}
		\bigl(F(\bar{\v{w}}_t^{(k)}) - F(\v{w}^\star)\bigr)\Biggr]\\
		\leq &
		\frac{2 L \|  \v{w}_1 - \v{w}^* \|^2}{\sqrt{K T }}
		+ \frac{\sigma_l^2}{4L N \sqrt{K T}}
		+ \frac{\sigma_l^2 }{4L T}
		+ \frac{9 \sigma_g^2 K }{8L T} \\
		& +\frac{ 2 L G^2}{ N \sqrt{KT} }\sum_{t=1}^{T}\sum_{n=1}^{N}\left( 1 - \beta_{n,t} ( 1 -q_{n,t}(h_{n,t},p_{n,t})) \right).
	\end{aligned}    
\end{equation}

\subsection{Proof of Theorem \ref{bellman}}
Let $J^{*}_{t}(\v{h}_t,\v{b}_t)$ denote the optimal objective value of the $(T-t)$-stage problem that starts from state $(\v{h}_t, \v{b}_t)$ at time slot $t$ and ends at time slot $T$, i.e.,
\begin{equation}
\begin{aligned}\label{def_J_t}
    J^{*}_{t}(\v{h}_t,\v{b}_t)=\min_{\pi}\mathbb{E}_{\pi}\left[\sum_{i=t}^{T}c(\v{h}_{i},\v{p}_{i}) | \v{h}_{t},\v{b}_{t}\right],\forall t.
\end{aligned} 
\end{equation}
Then, we prove by induction that $J^{*}_{t}(\v{h}_t,\v{b}_t)$ is equivalent to the value function $V_{t}(\v{h}_t,\v{b}_t)$.
It is obvious that $J^{*}_{T}(\v{h}_T,\v{b}_T)=V_{T}(\v{h}_T,\v{b}_T)$ according to the definition of $V_{T}(\v{h}_T,\v{b}_T)$ as shown in \eqref{bellman_equ_terminal}. Now, suppose that for some $t$, we have $J^{*}_{t+1}(\v{h}_{t+1},\v{b}_{t+1})=V_{t+1}(\v{h}_{t+1},\v{b}_{t+1})$. Then, we have
\begin{align}\label{J_induction}
	\nonumber
    &J^{*}_{t}(\v{h}_t,\v{b}_t)\\
    =&\min_{\pi}\mathbb{E}_{\pi}\left[\sum_{i=t}^{T}c(\v{h}_{i},\v{p}_{i}) | \v{h}_{t},\v{b}_{t}\right] \nonumber\\ 
    \nonumber
    =&\min_{\v{p}_{t}\in \ca{P}_{t}}\Bigg\{c(\v{h}_{t},\v{p}_{t})+\min_{\pi}\mathbb{E}_{\pi}\left[\sum_{i=t+1}^{T}c(\v{h}_{i},\v{p}_{i}) | \v{h}_{t+1},\v{b}_{t+1}\right]\Bigg\}\\
    \nonumber
    =&\min_{\v{p}_{t}\in \ca{P}_{t}}\left\lbrace  c(\v{h}_{t},\v{p}_{t})+J^{*}_{t+1}(\v{h}_{t+1},\v{b}_{t+1})   \right\rbrace \\
    \nonumber
    =&\min_{\v{p}_{t}\in \ca{P}_{t}} \left\lbrace  c(\v{h}_{t},\v{p}_{t})+V_{t+1}(\v{h}_{t+1},\v{b}_{t+1}) \right\rbrace  \\
    =&V_{t}(\v{h}_{t},\v{b}_{t}).
\end{align}  
Then, by induction, we have
\begin{equation}
\begin{aligned}\label{V_0}
    V_{1}(\v{h}_{1},\v{b}_{1})=&J^{*}_{1}(\v{h}_{1},\v{b}_{1})\\
    =&\min_{\pi}\mathbb{E}_{\pi}\left[\sum_{t=1}^{T} c(\v{h}_{t},\v{p}_{t}) | \v{h}_{1},\v{b}_{1}\right].
\end{aligned}    
\end{equation}
Now, we have proved that $V_{1}(\v{h}_{1},\v{b}_{1})$ is equivalent to the optimal objective value of Problem \ref{Problem1}.

\subsection{Proof of Lemma \ref{V_convex}}\label{proof_V_convex}

We first prove that the final-time value function $V_{T}(\v{h}_T,\v{b}_T)$ is convex in its battery state $\v{b}_T$. According to \cite[Theorem 9]{loyka2013convexity}, the packet error rate $q_{n,t}(h_{n,t},p_{n,t})$ is decreasing and convex in $p_{n,t}$. Then, we have that the cost $c(\v{h}_{t},\v{p}_{t})$ is decreasing and convex in $\v{p}_{t}$. Additionally, at the final time slot $T$, all the remaining energy in the battery should be used up for transmission, i.e., $b_{n,T}= e_{n}^{\text{cmp}}+\frac{p_{n,T} s_n}{r_{n,T}}$. Thus, we can obtain that the final-time value function $V_{T}(\v{h}_T,\v{b}_T)=\min_{\v{p}_T}\left\{c(\v{h}_T,\v{p}_T)\right\}$ is convex in $\v{b}_T$.
Now assume that $V_{t+1}(\v{h}_{t+1},\v{b}_{t+1}), 1 \leq t \leq T-1$, is convex in $\v{b}_{t+1}$ for given $\v{h}_{t+1}$. Then, we have that function
\begin{equation}\label{value_zhong}
	\begin{aligned}
		V_{t+1}(\v{h}_{t+1},\v{b}_{t+1}) = V_{t+1}(\v{h}_{t+1},\min\{\v{b}_{t} - \v{e}_t + \v{u}_t , \v{b}^{\textup{max}}\})
	\end{aligned}    
\end{equation}
is convex in $\v{b}_t$, where $\v{e}_t = \left\lbrace e_{1,t},\ldots, e_{N,t}\right\rbrace $ and $\v{u}_t = \left\lbrace u_{1,t},\ldots, u_{N,t}\right\rbrace $.
Since the expectation operator preserves convexity, the value function $V_{t}(\v{h}_{t},\v{b}_{t})$ in \eqref{bellman_equ} is the infimal convolution of two convex functions in $\v{b}_{t}$. Thus, the value function $V_{t}(\v{h}_{t},\v{b}_{t})$ is convex in $\v{b}_{t}$.

\subsection{Proof of Theorem \ref{p_decreasing_h}}

For clarity, given an action $\v{p}_{t}$, we define $\v{p}_{n,t}^+ = \{p_{1,t},\ldots,p_{n,t}+1,\ldots,p_{N,t}\}$. According to \cite[Theorem 1]{krishnamurthy2023interval}, to show that $p_{n,t}^*$ is non-increasing in $h_{n,t}$, it suffices to verify that the following four conditions hold:

(1) $c(\v{h}_{t},\v{p}_{t})$ is non-increasing in $h_{n,t}$.

(2) $\sum_{\left\lbrace\v{h}_{t+1},\v{b}_{t+1}| {h}_{n,t+1}\geq \tilde{h}\right\rbrace }\b{P}(\v{h}_{t+1},\v{b}_{t+1}|\v{h}_{t},\v{b}_{t},\v{p}_t)$ is non-decreasing in $h_{n,t}$, for all $\tilde{h}\in \ca{H}$.

(3) There exists a function $\gamma_{\v{h}_{t}^-,\v{h}_{t}^+,\v{p}_{t}}  > 0$ and non-decreasing in $\v{p}_t$ such that for any $\v{p}_t$ and $\v{h}_{t}^{-} \leq \v{h}_{t}^{+}$, we have $c(\v{h}_{t}^+,\v{p}_{n,t}^{+}) - c(\v{h}_{t}^+,\v{p}_{t}) \leq \gamma_{\v{h}_{t}^-,\v{h}_{t}^+,\v{p}_{t}}   [c(\v{h}_{t}^-,\v{p}_{n,t}^{+}) - c(\v{h}_{t}^-,\v{p}_{t})]$.

(4) $\sum_{\left\lbrace \v{h}_{t+1},\v{b}_{t+1}| {h}_{n,t+1}\geq \tilde{h}\right\rbrace }\b{P}(\v{h}_{t+1},\v{b}_{t+1}|\v{h}_{t},\v{b}_{t},\v{p}_t)$ is submodular in $h_{n,t}$ and $p_{n,t}$, for all $\tilde{h}\in \ca{H}$.

Note that condition (3) is less restrictive than the submodularity condition \eqref{submodular_equ}, and when $\gamma_{\v{h}_{t}^-,\v{h}_{t}^+,\v{p}_{t}}= 1$, condition (3) is equivalent to the submodularity condition.

First, it is clear that condition (1) holds due to the fact that the packet error rate $q_{n,t}(h_{n,t},p_{n,t})$ is decreasing in $h_{n,t}$.
Moreover, let $\gamma_{\v{h}_{t}^-,\v{h}_{t}^+,\v{p}_{t}} = \frac{c(\v{h}_{t}^+,\v{p}_{t}) - c(\v{h}_{t}^+,\v{p}^+_{n,t})}{c(\v{h}_{t}^-,\v{p}_{t}) } > 0$, and then we can obtain that condition (3) holds.

Next, we verify conditions (2) and (4).
We first define $\sum_{\left\lbrace\v{h}_{t+1},\v{b}_{t+1}| {h}_{n,t+1}\geq \tilde{h}\right\rbrace }\b{P}(\v{h}_{t+1},\v{b}_{t+1}|\v{h}_{t},\v{b}_{t},\v{p}_t)$ with fixed $\v{b}_t$ and $\left\lbrace h_{m,t},m\in [N]\backslash n\right\rbrace $ as follows:
	\begin{align}\label{sub_fucntion}
		G_{n,t}(h_{n,t},p_{n,t})=&\!\sum_{\left\lbrace\!\v{h}_{t+1},\v{b}_{t+1}| {h}_{n,t+1}\geq \tilde{h}\!\right\rbrace }\!\b{P}(\v{h}_{t+1},\v{b}_{t+1}|\v{h}_{t},\v{b}_{t},\v{p}_t) \nonumber \\
		\aeq&\!\sum_{\left\lbrace\! {h}_{n,t+1}\geq \tilde{h}\!\right\rbrace }\!\b{P}(h_{n,t+1}|h_{n,t}),
	\end{align} 
where (a) is due to \eqref{trans_prob}. According to Assumption \ref{channel_transmission_matrix_assumption}, we have that $G_{n,t}(h_{n,t},p_{n,t})$ is non-decreasing in $h_{n,t}$. Moreover, $G_{n,t}(h_{n,t},p_{n,t})$ is independent of $p_{n,t}$, and thus it is submodular in $h_{n,t}$ and $p_{n,t}$. Therefore, conditions (2) and (4) hold. Having confirmed conditions (1)-(4), it suffices to show that $p_{n,t}^*$ is non-increasing in $h_{n,t}$.

\subsection{Proof of Theorem \ref{asymptotically_optimal}}\label{proof_asymptotically_optimal}

Recall that $\overline \Omega_t$ denotes the set of scheduled devices under the optimal relaxed policy $\overline\pi^*$.
According to the truncation procedure, if $|\overline \Omega_t|>R$, $R$ devices will be chosen randomly from the set $\overline \Omega_t$ and scheduled.
Therefore, the probability that a device $n$ belongs to the set $\overline \Omega_t$ but is not scheduled is $\frac{(|\overline \Omega_t|-R)^+}{|\overline \Omega_t|}$, where $(\cdot)^+=\max\left\lbrace \cdot,0\right\rbrace $.
At time slot $t$, the per-device cost of device $n$ is no more than $\frac{ 2 L G^2  }{N \sqrt{KT}  }$.
Then, the gap between the cumulative costs under truncation policy $\tilde \pi$ and the optimal relaxed policy $\overline\pi^*$ is upper bounded by
\begin{align}\label{bound_J_thm_proof_1}
	J(\tilde \pi)-J(\overline \pi^*)
	\leq&\b{E}_{\overline \pi^*\!\!}\left[ \sum_{t=1}^{T}\sum_{n=1}^{N} \frac{2 L G^2  (|\overline \Omega_t|-R)^+}{N \sqrt{KT}|\overline \Omega_t|}\right] \nonumber\\
	=&\frac{ 2 L G^2  }{\sqrt{KT}  } \sum_{t=1}^{T}  \b{E}_{\overline \pi^*\!\!}\left[\frac{(|\overline \Omega_t|-R)^+}{|\overline \Omega_t|}\right]  \nonumber\\
	\aleq&\frac{ 2 L G^2  }{R \sqrt{KT}  }\sum_{t=1}^{T} \b{E}_{\overline \pi^*\!\!}\left[(|\overline \Omega_t|-R)^+\right]\nonumber\\
	\bleq&\frac{ 2 L G^2  }{R \sqrt{KT}  }\sum_{t=1}^{T} \b{E}_{\overline \pi^*\!\!}\left[ \left(|\overline \Omega_t|-\b{E}_{\overline \pi^*\!\!}\left[|\overline \Omega_t|\right]\right)^{ +}\right]\nonumber\\
	\leq&\frac{ 2 L G^2  }{R \sqrt{KT}  }\sum_{t=1}^{T}\underbrace{\b{E}_{\overline \pi^*\!}\!\left[ \left| |\overline \Omega_t|-\b{E}_{\overline \pi^*\!\!}\left[|\overline \Omega_t|\right]\right|\right]}_{=\textup{MAD}(|\overline \Omega_t|)},
\end{align}
where (a) follows from $\frac{(|\overline \Omega_t|-R)^+}{|\overline \Omega_t|}\leq \frac{(|\overline \Omega_t|-R)^+}{R}$, and (b) holds as $\b{E}_{\overline \pi^*}\left[|\overline \Omega_t|\right] \leq R$.
The notation $\textup{MAD}({X}):=\b{E}[|X - \b{E}[X]|]$ stands for the mean absolute deviation of its argument.

Let $\overline \beta_{n,t} $ denote the scheduling indicator under the optimal relaxed policy $\overline \pi^*$, and we have $|\overline \Omega_t|=\sum_{n=1}^N\overline \beta_{n,t}$. The scheduling indicators $\left\lbrace \overline \beta_{n,t}, n \in [N]\right\rbrace $ are independent binary random variables, since the optimal relaxed policy is obtained by decoupling the relaxed problem into per-device subproblems. Thus, we have that the standard deviation of $|\overline \Omega_t|$, $\textup{STD}(|\overline \Omega_t|)$, is upper bounded by
	\begin{align}\label{bound_J_thm_proof_add1}
		\textup{STD}(|\overline \Omega_t|)=\bigg(\sum_{n=1}^{N} \underbrace{\b{P}(\overline \beta_{n,t}=1)(1-\b{P}(\overline \beta_{n,t}=1))}_{\leq 1} \bigg)^{\frac{1}{2}}\leq \sqrt{\!N} .
	\end{align}
According to \cite{diaconis1991closed}, we have that $\textup{MAD}(X)\leq \textup{STD}(X)$ for arbitrary random variable $X$. Then, combining \eqref{bound_J_thm_proof_add1} into \eqref{bound_J_thm_proof_1}, we have
\begin{equation}\label{bound_J_thm_proof_add2}
	\begin{aligned}
		J(\tilde \pi)-J(\overline \pi^*)
		\leq  \frac{ 2 L G^2  \sqrt{NT} }{R \sqrt{K}  } =   \frac{ 2 L G^2  \sqrt{T} }{\kappa \sqrt{N K}  } .
	\end{aligned}    
\end{equation}

The expected cumulative cost under $\overline \pi^*$ is a lower bound of the expected cumulative cost under the optimal policy $\pi^*$, since constraint $\overline{\textup{C3}}$ is a relaxed version of constraint C3. Thus, we have
\begin{equation}\label{bound_J_thm_proof_2}
	\begin{aligned}
		J(\tilde \pi)-J(\pi^*)\leq& J(\tilde \pi)-J(\overline \pi^*)\leq \frac{ 2 L G^2  \sqrt{T} }{\kappa \sqrt{N K}  } .
	\end{aligned}    
\end{equation}

\bibliography{IEEEabrv,PA_DMTL}
\bibliographystyle{ieeetr}

\end{document}